\documentclass{article}

\PassOptionsToPackage{numbers, compress}{natbib}

\usepackage[final]{neurips_2023}

\usepackage[utf8]{inputenc} 
\usepackage[T1]{fontenc}    
\usepackage{hyperref}       
\usepackage{url}            
\usepackage{booktabs}       
\usepackage{amsfonts}       
\usepackage{nicefrac}       
\usepackage{microtype}      
\usepackage{xcolor}         
\usepackage{amsmath}
\usepackage{amssymb}
\usepackage{amsthm}

\usepackage{url}
\usepackage{microtype}
\usepackage{graphicx}
\usepackage{booktabs} 
\usepackage{colortbl}
\usepackage{tabulary}
\usepackage{adjustbox}
\usepackage{multirow}
\usepackage{makecell} 
\usepackage{caption}
\usepackage{subcaption}
\usepackage{enumitem}
\usepackage{xcolor}

\usepackage{cleveref}
\usepackage{wrapfig}
\usepackage{stackengine}
\usepackage{comment}
\usepackage{mathtools}

\usepackage[ruled]{algorithm2e}

\usepackage{xcolor}
\hypersetup{
    colorlinks,
    linkcolor={red!50!black},
    citecolor={blue!50!black},
    urlcolor={blue!80!black}
}

\usepackage{bm, dsfont}

\let\hat\widehat
\let\tilde\widetilde

\newcommand{\vb}{\mathbf{v}}

\newcommand{\bu}{\bm{u}}
\newcommand{\bv}{\bm{v}}

\newcommand{\bx}{\bm{x}}
\newcommand{\by}{\bm{y}}
\newcommand{\bz}{\bm{z}}

\newcommand{\bH}{\bm{H}}
\newcommand{\bI}{\bm{I}}

\newcommand{\cA}{\mathcal{A}}
\newcommand{\cB}{\mathcal{B}}

\newcommand{\cD}{\mathcal{D}}

\newcommand{\cL}{\mathcal{L}}

\newcommand{\cN}{\mathcal{N}}
\newcommand{\cO}{\mathcal{O}}

\newcommand{\RR}{\mathbb{R}}

\newcommand{\btheta}{\bm{\theta}}

\newcommand{\bSigma}{\bm{\Sigma}}

\newcommand{\tr}{\mathop{\mathrm{tr}}}

\newcommand{\diag}{{\rm diag}}

\newcommand{\norm}[1]{\left\|#1\right\|}
\def\abs#1{\left| #1 \right|}

\newcommand*{\E}{\mathbb{E}}

\def\vb{{\bm{b}}}

\def\vd{{\bm{d}}}

\makeatletter
\newcommand{\ve}{\@ifnextchar\bgroup{\velong}{{\bm{e}}}}
\newcommand{\velong}[1]{{\bm{#1}}}
\makeatother

\def\vg{{\bm{g}}}

\def\vm{{\bm{m}}}

\def\vv{{\bm{v}}}

\def\vx{{\bm{x}}}
\def\vy{{\bm{y}}}
\def\vz{{\bm{z}}}
\def\vtheta{{\bm{\theta}}}

\def\mA{{\bm{A}}}
\def\mB{{\bm{B}}}

\def\mI{{\bm{I}}}

\def\mW{{\bm{W}}}

\newcommand{\Sym}{{\rm Sym}}

\newcommand{\Loss}{\mathcal{L}}

\newcommand{\mf}[1]{\mathbf{#1}}
\newcommand{\mbb}[1]{\mathbb{#1}}
\newcommand{\tf}[1]{\textbf{#1}}

\newcommand{\sgd}{{SGD}}
\newcommand{\zosgd}{{ZO-SGD}}
\newcommand{\zosgds}{{ZO}}

\newcommand{\etazo}{\eta_{\text{ZO}}}
\newcommand{\etasgd}{\eta_{\text{SGD}}}

\newcommand{\ft}{{FT}}

\newcommand{\mezo}{MeZO}

\usepackage{tikz}
\newcommand{\xmark}{%
\tikz[scale=0.23] {
    \draw[line width=0.7,line cap=round,red] (0,0) to [bend left=6] (1,1);
    \draw[line width=0.7,line cap=round,red] (0.2,0.95) to [bend right=3] (0.8,0.05);
}}

\newcommand{\cmark}{%
\tikz[scale=0.23] {
    \draw[line width=0.7,line cap=round,green!70!black] (0.25,0) to [bend left=10] (1,1);
    \draw[line width=0.8,line cap=round,green!70!black] (0,0.35) to [bend right=1] (0.23,0);
}}

\newcommand{\enote}[1]{{\color{blue}[EN: #1]}}
\newcommand{\snote}[1]{{\color{brown}[SM: #1]}}

\newcommand{\tnote}[1]{{\color{cyan}[TG: #1]}}
\newcommand{\todo}[1]{{\color{red}[TODO: #1]}}

\theoremstyle{plain}
\newtheorem{theorem}{Theorem}
\newtheorem{lemma}{Lemma}
\newtheorem{corollary}{Corollary}
\newtheorem{definition}{Definition}
\newtheorem{assumption}{Assumption}
\newtheorem{proposition}{Proposition}

\theoremstyle{remark}
\newtheorem{remark}{Remark}

\title{Fine-Tuning Language Models with Just \\Forward Passes}

\author{%
  Sadhika Malladi\thanks{Equal contribution and corresponding authors.} \\
  \And
  Tianyu Gao$^{\ast}$ \\
  \And
  Eshaan Nichani \\
  \And 
  Alex Damian \\
  \AND 
  Jason D. Lee \\
  \And
  Danqi Chen \\
  \And
  Sanjeev Arora \\
}

\ifdefined\usebigfont

\usepackage{times}
\usepackage[fontsize=13pt]{scrextend}
\makeatletter
\@ifpackageloaded{geometry}{\AtBeginDocument{\newgeometry{letterpaper,left=1.56in,right=1.56in,top=1.71in,bottom=1.77in}}}{\usepackage[letterpaper,left=1.56in,right=1.56in,top=1.71in,bottom=1.77in]{geometry}}
\makeatother
\else
\fi

\begin{document}
\maketitle

\setcounter{footnote}{0}


\begin{abstract}
Fine-tuning language models (LMs) has yielded success on diverse downstream tasks, but as LMs grow in size, backpropagation requires a prohibitively large amount of memory. 
Zeroth-order (ZO) methods can in principle estimate gradients using only two forward passes but are theorized to be catastrophically slow for optimizing large models. 
In this work, 
we propose a memory-efficient zeroth-order optimizer (\tf{\mezo{}}), 
adapting the classical \zosgd{} method to operate in-place, thereby fine-tuning LMs with \textit{the same memory footprint as inference}.
For example, with a single A100 80GB GPU, 
\mezo{} can train a 30-billion parameter model, whereas fine-tuning with backpropagation 
can train only a 2.7B LM with the same budget. 
We conduct comprehensive experiments across model types (masked and autoregressive LMs), model scales (up to 66B), and downstream tasks (classification, multiple-choice, and generation).
Our results demonstrate that 
(1) \mezo{} significantly outperforms in-context learning and linear probing;
(2) \mezo{} achieves comparable performance to fine-tuning with backpropagation across multiple tasks, 
with up to 12$\times$ memory reduction and up to 2$\times$ GPU-hour reduction in our implementation; 
(3) \mezo{} is compatible with both full-parameter and parameter-efficient tuning techniques such as LoRA and prefix tuning; 
(4) \mezo{} can effectively optimize non-differentiable objectives (e.g., maximizing accuracy or F1).
We support our empirical findings with theoretical insights, 
highlighting how adequate pre-training and task prompts enable \mezo{} to fine-tune huge models, despite classical ZO analyses suggesting otherwise.\footnote{Our code is available at \url{https://github.com/princeton-nlp/MeZO}.}
\end{abstract}


\section{Introduction}
Fine-tuning  pre-trained language models (LMs) has been the dominant methodology for 
solving many language tasks~\citep{devlin-etal-2019-bert}, 
adapting to specialized domains~\citep{gururangan-etal-2020-dont}, 
or incorporating human instructions and preferences~\citep{ouyang2022training}.
However, as LMs are scaled up~\citep{brown2020language,openai2023gpt4}, computing gradients for backpropagation requires a prohibitive amount of memory 
-- in our test, up to $12\times$ the memory required for inference --
because it needs to cache activations during the forward pass, gradients during the backward pass, and, in the case of Adam~\citep{kingma2014adam}, also store gradient history~(see~\Cref{sec:memory} for a detailed analysis).
  
As a result, while it is possible to run inference with a 30-billion (30B) parameter LM on 
a single Nvidia A100 GPU (with 80GB memory),
backpropagation with Adam is feasible only for a 2.7B LM. 
Parameter-efficient fine-tuning methods (PEFT~\citep{hu2021lora,li-liang-2021-prefix,lester-etal-2021-power}) update just a fraction of the network parameters, but 
 still need to cache many activations, because the tuned parameters are scattered throughout the model. 
In our tests, fine-tuning an OPT-13B model with full parameter tuning or PEFT requires 12$\times$ and $6\times$ more memory than inference respectively.

\begin{figure}
    \center
    \includegraphics[width=0.99\textwidth]{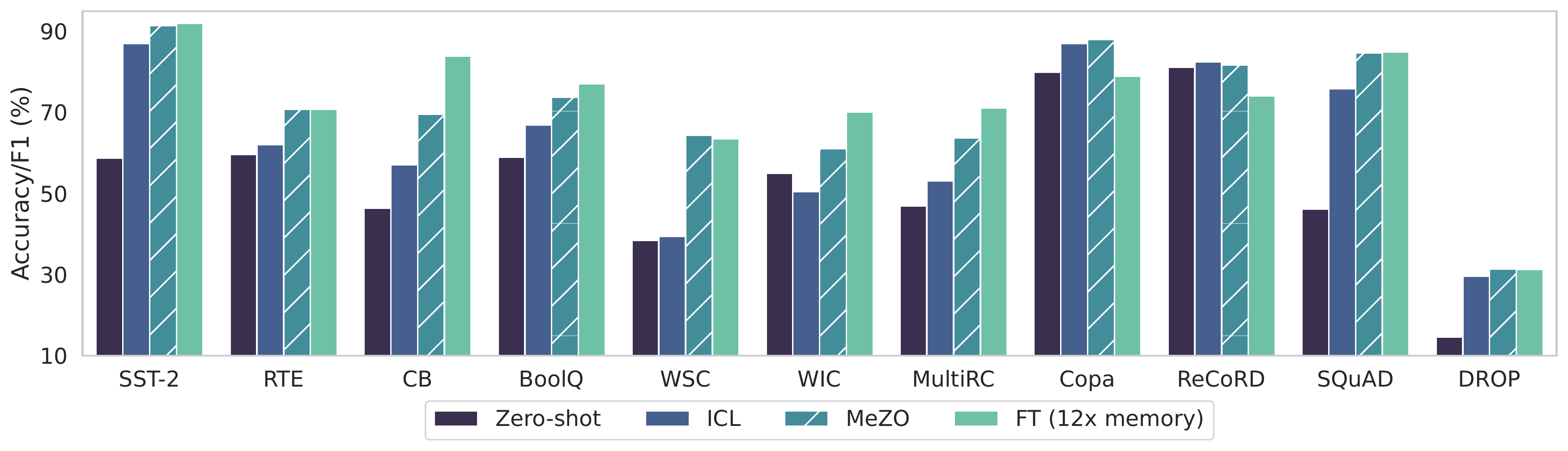}
    \caption{OPT-13B results with zero-shot, in-context learning (ICL), \mezo{} (we report the best among \mezo{}/\mezo{} (LoRA)/\mezo{} (prefix)), and fine-tuning with Adam (FT). 
    \mezo{} demonstrates superior results over zero-shot and ICL and performs on par with     FT (within 1\%) on 7 out of 11 tasks, despite using only 1/12  memory.
    See Table~\ref{tab:opt} for detailed numbers and Figure~\ref{fig:memory_fig} for memory profiling.
    }
    \vspace{-10pt}
    \label{fig:teaser}
\end{figure}

\looseness-1
\emph{In-context learning} (ICL~\citep{brown2020language}) has allowed solving many tasks with a single inference pass, during which the model processes labeled examples (\textit{demonstrations}) in its context and then outputs a prediction on a test example. 
While this allows for quick adaptation of the model to specific use cases, 
 current models allow a limited context size (and thus, limited demonstrations) and the performance is sensitive to the formatting and choice of demonstrations~\citep{liu-etal-2022-makes,lu-etal-2022-fantastically}.
ICL can slow with the number of demonstrations, and it often performs worse than fine-tuning of medium-sized models~\citep{brown2020language}. 

Backpropagation
also cannot optimize non-differentiable criteria, which have gained popularity 
in fine-tuning LMs according to human preference scores or set safety standards
\citep{stiennon2020learning,ouyang2022training}. 
Typically,
these adaptations involve expensive reinforcement learning from human feedback (RLHF~\citep{christiano2017deep}).



A classical zeroth-order optimization method, \zosgd{}~\citep{spall1992multivariate}, uses only differences of loss values to 
estimate the gradients.
Thus, in principle, the method can update neural networks with just forward passes,
though naive implementation still doubles the memory overhead and classical lower bounds~\citep{nemirovskij1983problem,duchi2015optimal} suggest that convergence slows linearly with model size. 
As such, ZO methods have been applied in deep learning settings to find adversarial examples or tune input embeddings~\citep{sun2022black,sun2022bbtv2} but not to directly optimize large-scale models (see~\citet{liu2020understanding} for a survey).


In this work, we propose a memory-efficient zeroth-order optimizer (\mezo{}),
which adapts the classical \zosgd{} algorithm and reduces its memory consumption \textit{to the same as inference}.
We apply \mezo{} to fine-tune large LMs and show that, both empirically and theoretically,
\mezo{} can successfully optimize LMs with billions of parameters. 
Specifically, our contributions are:
\vspace{-5pt}
\begin{enumerate}[leftmargin=2em,itemsep=1pt]
	\item In \mezo{}, we adapt the \zosgd{} algorithm~\citep{spall1992multivariate} and a number of variants to operate in-place on arbitrarily large models with almost no memory overhead (see Algorithm~\ref{alg:zo_sgd} and~\Cref{sec:prelims}).
    \item 
    We conduct comprehensive experiments across model types (masked LM and autoregressive LM), 
    model scales (from 350M to 66B), 
    and downstream tasks (classification, multiple-choice, and generation). 
    \mezo{} consistently outperforms zero-shot, ICL, and linear probing.
    Moreover, 
    with RoBERTa-large, \mezo{} achieves performance close to standard fine-tuning within 5\%  gap; 
    with OPT-13B,  
    \mezo{}
    outperforms or performs comparably to fine-tuning on 7 out of 11 tasks, despite 
    requiring roughly $12\times$ less memory (\Cref{fig:teaser} and \Cref{sec:exp}).
    In our implementation, MeZO requires only half as many GPU-hours as Adam fine-tuning for a 30B model (see~\Cref{sec:time}). 
    \item
    We demonstrate \mezo{}'s compatibility with full-parameter tuning and PEFT (e.g., LoRA~\citep{hu2021lora} and prefix-tuning~\citep{li-liang-2021-prefix}) in \Cref{sec:exp}.
	\item Further exploration showcases that \mezo{} can optimize non-differentiable objectives such as accuracy or F1 score, while still requiring only the same memory as inference (\Cref{sec:nondiff}). 
	\item Our theory suggests that adequate pre-training ensures the per-step optimization rate (\Cref{thm:rate_comparison}) and global convergence rate (\Cref{lem:global_ZO-SGD}) of \mezo{} depend on a certain condition number of the landscape (i.e., the local effective rank, see \Cref{assume:low_eff_rank}) instead of numbers of parameters. 
    This result is in sharp contrast to existing ZO lower bounds~\citep{nemirovskij1983problem,duchi2015optimal} suggesting that the convergence rate can slow proportionally to the number of parameters (\Cref{sec:theory}). 
\end{enumerate}

\section{Zeroth-order optimization}\label{sec:prelims}

\begin{figure}[t] 
\centering
\begin{algorithm}[H]
  \SetKwFunction{perturb}{PerturbParameters}
  \SetKwProg{sub}{Subroutine}{}{}
  \SetKwComment{Comment}{$\triangleright$\ }{}
  
  \textbf{Require}: parameters $\vtheta\in\RR^d$, 
  loss $\cL:\RR^d\to\RR$, 
  step budget $T$,
  perturbation scale $\epsilon$, 
  batch size $B$,
  learning rate schedule $\{\eta_t\}$ \\
  \vspace{0.2cm}
  \For{$t=1,...,T$} { 
    Sample batch $\cB\subset \cD$ and random seed $s$ \\
    $\vtheta\gets$ \perturb{$\vtheta, \epsilon, s$} \\
    $\ell_+\gets\cL(\vtheta;\cB)$ \\
    $\vtheta\gets$ \perturb{$\vtheta, -2\epsilon, s$} \\
    $\ell_-\gets\cL(\vtheta;\cB)$ \\
    $\vtheta\gets$ \perturb{$\vtheta, \epsilon, s$} \Comment*[f]{Reset parameters before descent}
    \BlankLine
    $\texttt{projected\_grad}\gets (\ell_+ - \ell_-) / (2\epsilon)$ \\
    Reset random number generator with seed $s$ \Comment*[f]{For sampling $z$} \\
    \For{$\theta_i\in\vtheta$}{
        $z\sim\cN(0,1)$ \\
        $\theta_i\gets\theta_i - \eta_t * \texttt{projected\_grad} * z$  
    }
  } 
  
  \vspace{0.2cm}
  \sub{\perturb{$\vtheta$, $\epsilon$, $s$}}{
  Reset random number generator with seed $s$  \Comment*[f]{For sampling $z$} \\
    \For{$\theta_i\in\vtheta$ }{  
        $z\sim\cN(0,1)$ \\
        $\theta_i\gets\theta_i+\epsilon z$ \Comment*[f]{Modify parameters in place} 
    }
    \Return $\vtheta$
  }
  
  \caption{\mezo{}}

  \label{alg:zo_sgd}
\end{algorithm}
\vspace{-10pt}
\end{figure}

Zeroth-order (ZO) optimizers have long been studied in the context of convex and strongly convex objectives.
In the following, we first introduce 
a classical ZO gradient estimator, SPSA (\Cref{def:spsa}~\citep{spall1992multivariate}) and
the corresponding SGD algorithm, \zosgd{} (\Cref{def:zo_sgd}).
Then we describe \mezo{}, our in-place implementation that requires the same memory as inference in \Cref{sec:memory_efficient_zo} and Algorithm~\ref{alg:zo_sgd}.
We highlight that SPSA can also be used in more complex optimizers, such as Adam, and we provide memory-efficient implementations for those algorithms too (\Cref{sec:ext_zo}).

Consider a labelled dataset $\cD = \{(\bx_i, \by_i)\}_{i \in [\abs{\cD}]}$  
and a minibatch $\cB \subset \cD$ of size $B$, we let $\cL(\btheta; \cB)$ denote the loss on the minibatch. 
We introduce a classical ZO gradient estimate in this setting.\footnote{The original SPSA algorithm~\citep{spall1992multivariate} perturbs the model by $1/\vz$ and thus requires that $\vz$ has finite inverse moments, precluding the choice of $\vz$ as Gaussian. $1/\vz$ is very large with high probability for a zero-mean Gaussian $\vz$, so we adopt the standard in many theoretical~\citep{nesterov2017random,duchi2015optimal} and empirical~\citep{liu2020primer} works and perturb the parameters by $\vz$ with $\vz$ as a Gaussian random variable.}
\begin{definition}[Simultaneous Perturbation Stochastic Approximation or SPSA~\citep{spall1992multivariate}]
	Given a model with parameters $\vtheta\in\RR^d$ and a loss function $\Loss$, SPSA estimates the gradient on a minibatch $\cB$ as
	\begin{equation}
		\hat\nabla\cL(\vtheta;\cB) =  \frac{\cL(\vtheta + \epsilon\vz;\cB) - \cL(\vtheta - \epsilon\vz;\cB)}{2\epsilon}\vz \approx \vz\vz^\top \nabla\cL(\vtheta;\cB) 
		\label{eq:spsa}
	\end{equation}
	where $\vz\in\RR^d$ with $\vz\sim\cN(0,\mI_d)$ and $\epsilon$ is the \emph{perturbation scale}. The $n$-SPSA gradient estimate averages $\hat\nabla\cL(\vtheta;\cB)$ over $n$ randomly sampled $\vz$. 
	\label{def:spsa}
\end{definition}
SPSA requires only \textit{two forward passes} through the model to compute the gradient estimate (for $n$-SPSA, each estimate requires $2n$ forward passes). 
As $\epsilon\to 0$, the SPSA estimate can be understood as a rank-1 reconstruction of the gradient.
During training, $n$ can be treated as a hyperparameter and follow a schedule~\citep{bollapragada2018adaptive,cai2022zoro},
though in cursory experiments (\Cref{app_sec:ablations}), $n=1$ is the most efficient. We use $n=1$ as the default. 
It is widely known that the SPSA estimate can be used to replace the backpropagation gradient in any optimizer such as SGD.
\begin{definition}[ZO-SGD]
	ZO-SGD is an optimizer with learning rate $\eta$ that updates parameters as $\vtheta_{t+1} = \vtheta_t - \eta\hat\nabla\cL(\vtheta;\cB_t)$ where $\cB_t$ is the minibatch  at time $t$ and $\hat\nabla\cL$ is the SPSA gradient estimate.
    \label{def:zo_sgd} 
\end{definition} 

\subsection{Memory-efficient \zosgd{} (\mezo{})}\label{sec:memory_efficient_zo}

The vanilla \zosgd{} algorithm costs twice the memory of inference,
as it needs to store $\vz\in \mathbb{R}^d$. 
We propose a memory-efficient implementation of \zosgd{} called \tf{\mezo{}}, 
as illustrated in Algorithm~\ref{alg:zo_sgd}.
At  each step,
we first sample a random seed $s$, 
and then for each of $\vz$'s four uses in Algorithm~\ref{alg:zo_sgd}, 
we reset the random number generator by $s$ 
and \textit{resample} the relevant entry of $\vz$.
Using this in-place implementation, \mezo{} has a memory footprint equivalent to the inference memory cost.

We note that Algorithm~\ref{alg:zo_sgd} describes perturbing each parameter separately, which may be time-consuming for large models.
In practice, we can save time by perturbing an entire weight matrix instead of each scalar independently.
This incurs an additional memory cost as large as the largest weight matrix; usually, this is the word embedding matrix (e.g., 0.86GB for OPT-66B).

\paragraph{Storage Efficiency of \mezo{}.}
Parameter-efficient fine-tuning (PEFT) techniques fine-tune just a fraction of the network parameters and have thus been proposed as a way to reduce the storage costs of fine-tuned model checkpoints. 
Fine-tuning with \mezo{} reduces the storage cost of the resulting checkpoint far more than popular PEFT techniques (e.g., LoRA~\citep{hu2021lora} and prefix tuning~\cite{li-liang-2021-prefix}).
We reconstruct the \mezo{} trajectory using a single seed, which spawns step-wise seeds to sample $\vz$, and the \texttt{projected\_grad} at each step.\footnote{Note that this reconstruction requires no additional forward passes through the model and no access to the data used during fine-tuning, since \texttt{projected\_grad} implicitly encodes this information.}
As such, for fine-tuning a 66B model, 
\mezo{} requires saving the seed plus 20,000 (steps) $\times$ 
2 bytes, which is less than 0.1MB.
LoRA fine-tunes 19M parameters and requires 38MB storage, and prefix tuning fine-tunes 6M parameters and requires 12MB storage. 

\subsection{\mezo{} extensions}
\label{sec:ext_zo}
We note that SPSA is a popular ZO gradient estimator but not the only one. 
Many one-point gradient estimators have been proposed in past works~\citep{flaxman2005online,spall1997one,vakhitov2009algorithm}, and using such estimators in place of SPSA would halve the training time.
However, cursory experiments with one such promising estimator ~\citep{zhang2022new} reveal that these are not as efficient as SPSA when fixing the number of forward passes (\Cref{sec:one_point_estimate}). 
As such, we implement MeZO with the SPSA estimator.

\mezo{} can also be combined with other gradient-based optimizers, including SGD with momentum or Adam.
Though naive implementation would require additional memory to store the gradient moment estimates,
\mezo{}-momentum and \mezo{}-Adam alleviate such overhead by
recomputing the moving average of the gradients using saved
past losses and $\vz$ (see \Cref{sec:alg_variants} for a full discussion).


We also note that all of the coordinates of the SPSA gradient estimate have the same scale, but deep Transformers can have gradients of different scales for each layer~\citep{li2022robust,liu2020understanding}.
As such, we draw inspiration from layerwise adaptive optimizers~\citep{you2017large,you2019large} 
to design several \mezo{} variants.
Cursory experiments showed that these algorithms are not more efficient (in terms of forward passes), but we nevertheless present them as potential optimizers for more complex objectives.
See \Cref{sec:alg_variants}.

\paragraph{Forward Auto-Differentiation}
Note that $\vz^\top\nabla\cL(\vtheta;\cB)$ is a Jacobian-vector product (JVP), which can be computed in parallel with an inference pass with excess memory consumption equivalent to that of the largest activation in the network~\citep{griewank2008evaluating}.
In this case, $\vz$ must be stored on the GPU in order to construct the gradient estimate, so this procedure requires slightly more than two times the memory needed for inference.
We analyze this algorithm in detail in~\Cref{app_sec:fwd_ad}. Note that using a non-zero $\epsilon$ in SPSA, which is not possible through the JVP method, may boost generalization by promoting a sharpness-minimizing term. 
Past works (e.g., \citet{baydin2022gradients}) have also studied JVP-based training but achieved limited empirical success.

\begin{figure}[t]
    \centering 
    \includegraphics[width=0.99\textwidth]{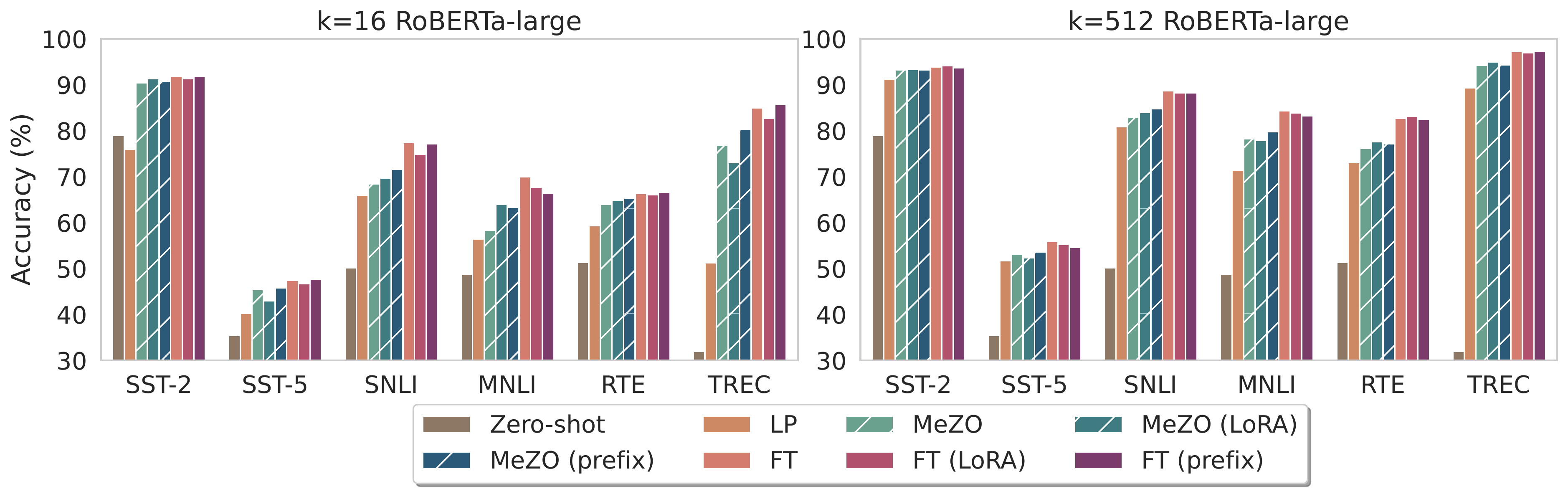}
    \caption{
        Experiments on RoBERTa-large. 
        We report zero-shot, linear probing (LP), and \mezo{} and fine-tuning (FT) with full parameter, LoRA, and prefix-tuning. 
        \mezo{} outperforms zero-shot and LP  and approaches FT  (within 5\% for $k=512$) with much less memory. Detailed numbers in \Cref{tab:roberta}.
        }  
        \vspace{-15pt}
    \label{fig:rob} 
\end{figure} 

\section{Experiments}
\label{sec:exp}


Preliminary experiments (\Cref{app_sec:ablations})  
show that \mezo{} only works when using prompts~\citep{brown2020language,schick-schutze-2021-exploiting,gao-etal-2021-making}.
Past works~\cite{saunshi2021a,malladi2023kernelbased} have demonstrated how the inclusion of a suitable prompt ensures the fine-tuning objective is closely related to the pre-training one.
In \Cref{sec:theory}, we extend these ideas to show how using a simple prompt simplifies the fine-tuning optimization procedure, thereby enabling zeroth order methods to work efficiently.
All experiments below 
use prompts detailed in \Cref{app_sec:our_prompt}.
All fine-tuning with backpropagation (FT) experiments follow convention and use Adam, though we also report results when performing FT with SGD in \Cref{app_sec:more_exp}.


We conduct comprehensive experiments on 
both medium-sized masked LMs (RoBERTa-large, 350M~\citep{liu2019roberta}) and
large autoregressive LMs (OPT-13B, 30B, 66B~\citep{zhang2022opt}) in 
few-shot and many-shot settings with prompts.
We also explore both full-parameter tuning and PEFT including LoRA~\citep{hu2021lora} and prefix-tuning~\citep{li-liang-2021-prefix} (see Appendix~\ref{sec:peft} for details). 
We compare \mezo{} with zero-shot, in-context learning (ICL), linear-probing (LP), and fine-tuning with Adam (FT).
\mezo{} uses substantially less memory than FT but requires significantly more training steps.

We first
show that \mezo{} improves substantially over zero-shot, ICL, and LP across model types, sizes, and task types. 
Moreover, \mezo{} performs 
comparably to FT over a number of tasks, 
while drastically reducing the memory cost by, for example, 12$\times$ on OPT-13B.
Further experiments demonstrate that \mezo{} can optimize non-differentiable objectives, such as accuracy and F1 score (\Cref{sec:nondiff}).
We compare the memory consumption of ICL, FT, LP, and \mezo{} in \Cref{fig:memory_fig,tab:memory_tab}.

\subsection{Medium-sized masked language models}

We conduct experiments with RoBERTa-large
on sentiment classification, natural language inference, and topic classification tasks.
We follow past works~\citep{gao-etal-2021-making,malladi2023kernelbased} in studying the few-shot and many-shot settings, sampling $k$ examples per class for $k=16$ 
and $k=512$
(details in  \Cref{app_sec:expsetup}).
We run \mezo{} for $100$K steps and fine-tuning for $1000$ steps, noting that one \mezo{} step is substantially faster than one fine-tuning step (see~\Cref{sec:time} for a comparison).
We summarize the results from \Cref{fig:rob} and \Cref{tab:roberta} below.

\paragraph{\mezo{} works significantly better than zero-shot, linear probing, and other memory-equivalent methods.}
On all six diverse tasks, 
\mezo{} can optimize the pre-trained model and consistently perform better than zero-shot and linear probing.
We also show for several tasks that \mezo{} can outperform another ZO algorithm, BBTv2~\citep{sun2022bbtv2}, by up to $11\%$ absolute (\Cref{app_sec:bbtv2}).\footnote{BBTv2 can only train low-dimensional projected prefixes instead of the full model.} 


\paragraph{With enough data, \mezo{} achieves comparable performance (up to 5\% gap) to FT.}
\mezo{} achieves close-to-fine-tuning performance on $k=16$, with some tasks only having 2\% gaps.
When using $k=512$ data,
the gap between \mezo{} and \ft{} further reduced to within 5\% across all tasks.


\paragraph{\mezo{} works well on both full-parameter tuning and PEFT.}
Full-parameter tuning (\mezo{}) and 
PEFT (\mezo{} with LoRA and prefix-tuning) 
achieve comparable performance, while \mezo{} (prefix) sometimes outperforms \mezo{}.
We also show in \Cref{app_sec:conv_zo_full_peft} that the three variants converge at similar rates, agreeing with our theory in \Cref{sec:theory}, which shows that
\mezo{} converges at a rate independent of the number of parameters being optimized. 


We show additional results with more FT and \mezo{} variants in Appendix~\ref{app_sec:roberta_more}. 
We see that (1) ZO-Adam sometimes outperforms ZO-SGD but is not consistent across tasks; (2) LP and then \mezo{}, as suggested for fine-tuning~\citep{kumar2022finetuning}, can sometimes improve the performance. 


\begin{table*}[t]
\centering
\setlength{\tabcolsep}{4pt}
\resizebox{\textwidth}{!}{
    \begin{tabular}{lccccccccccc}
    \toprule
     Task  & \tf{SST-2}	& \tf{RTE} & \tf{CB} & \tf{BoolQ} & \tf{WSC} & \tf{WIC}	& \tf{MultiRC} & \tf{COPA} & \tf{ReCoRD} & \tf{SQuAD} & \tf{DROP} \\
    Task type & \multicolumn{7}{c}{------------------------ classification ------------------------} & \multicolumn{2}{c}{-- multiple choice --} & \multicolumn{2}{c}{--- generation ---}\\
    \midrule
    Zero-shot & 58.8 & 59.6 & 46.4 & 59.0 &	38.5 & 55.0	& 46.9 & 80.0& 81.2& 46.2 &14.6\\
    ICL & 87.0 & 62.1 &	57.1 & 66.9	& 39.4 & 50.5 & 53.1 & 87.0&  \tf{82.5}& 75.9 & 29.6\\
    LP & \tf{93.4}&	68.6&	67.9&	59.3&	63.5&	60.2&	63.5& 55.0 &27.1& 3.7 & 11.1\\
    \midrule
    \mezo{}        & 91.4    &66.1	&67.9	&67.6	&63.5	&\tf{61.1}	&60.1 & \tf{88.0}& 81.7& \tf{84.7} & 30.9\\
    \mezo{} {(LoRA)} & 89.6	&67.9	&66.1	&\tf{73.8}	&\tf{64.4}	&59.7	&61.5 &84.0&81.2& 83.8 & \tf{31.4}\\
    \mezo{} (prefix) & 90.7	&\tf{70.8}	& \tf{69.6}&	73.1	&60.6& 59.9	&\tf{63.7} &87.0& 81.4& 84.2 & 28.9\\
    \midrule
     FT {\fontsize{8}{9.6}\selectfont (12x memory)} & {92.0} & {70.8} &	{83.9} &	{77.1}	& {63.5} &	{70.1} &71.1 &{79.0} & {74.1} &{84.9}	& {31.3} \\
    
    \bottomrule
    \end{tabular}}
    \caption{
        Experiments on OPT-13B (with $1000$ examples). ICL: in-context learning; LP: linear probing; FT: full fine-tuning with Adam.  
        \mezo{} outperforms zero-shot, ICL, and LP across the board, and achieves comparable (within 1\%) or better performance than FT on 7 out of 11  tasks. 
    }
    \label{tab:opt}
\end{table*}

\begin{table}[t]
\centering
\resizebox{0.75\textwidth}{!}{
    \begin{tabular}{lcccccc}
    \toprule
     Task  & \tf{SST-2} & \tf{RTE} & \tf{BoolQ} & \tf{WSC} & \tf{WIC} & \tf{SQuAD} \\
    \midrule
    30B zero-shot & 56.7 & 52.0 & 39.1 & 38.5 &	50.2 & 46.5\\
    30B ICL & 81.9 & 66.8 & 66.2 &56.7&	51.3 & 78.0\\
    30B \mezo{}/\mezo{} (prefix) &  \tf{90.6} & \tf{72.6} & \tf{73.5} & \tf{63.5} & \tf{59.1} & \tf{85.2}\\
    \midrule
    66B zero-shot & 57.5 & \tf{67.2} & 66.8&	43.3& 50.6&48.1\\
    66B ICL & 89.3 & 65.3&	62.8&	52.9&54.9&81.3\\
    66B \mezo{}/\mezo{} (prefix) & \tf{93.6} & 66.4 & \tf{73.7} & \tf{63.5} &\tf{58.9} & \tf{85.0} \\
    \bottomrule
    \end{tabular}}
    \vspace{5pt}
    \caption{
        Experiments on OPT-30B and OPT-66B (with $1000$ examples). 
        We report the best of \mezo{} and \mezo{} (prefix). See \Cref{app_sec:opt_more} for more results.
        We see that on most tasks 
        \mezo{} effectively optimizes up to 66B models and outperforms zero-shot and ICL. 
    }
    \vspace{-15pt}
    \label{tab:large}
\end{table}

\subsection{Large autoregressive language models}

With the promising results from RoBERTa-large, 
we extend \mezo{} to the OPT family~\citep{zhang2022opt}, on a scale of 
13B (Table~\ref{tab:opt}), 30B, and 66B (\Cref{tab:large}).
We select both SuperGLUE~\citep{wang2019superglue} tasks\footnote{We also include SST-2, which is a simple sentiment classification task that we use for development.} (including classification and multiple-choice) and generation tasks.
We randomly sample $1000$, $500$, and $1000$ examples for training, validation, and test, respectively, for each datset.
We run \mezo{} for $20$K steps and fine-tuning for $5$ epochs, or $625$ steps, noting that each step of \mezo{} is substantially faster than fine-tuning (see~\Cref{sec:time} for a comparison).
Please refer to Appendix~\ref{app_sec:expsetup} for details.
Table~\ref{tab:opt} yields the following observations. 

\paragraph{\mezo{} outperforms memory-equivalent methods and closely approaches fine-tuning results.}
We see that on a 13B-parameter scale, \mezo{} and its PEFT variants outperform zero-shot, ICL, and LP across almost all tasks. 
When comparing to FT, which costs 12$\times$ more memory (\Cref{sec:memory}), 
\mezo{} achieves comparable (within 1\%) or better performance  on 7 out of the 11 tasks.


\paragraph{\mezo{} exhibits strong performance across classification, multiple-choice, and generation tasks.}
We investigate \mezo{} on generation tasks, which are  regarded as more intricate than classification or multiple-choice tasks. 
We evaluate on two question answering datasets, SQuAD~\citep{rajpurkar-etal-2016-squad}
and DROP~\citep{dua-etal-2019-drop}. 
We use teacher forcing for training and greedy decoding for inference (details in \Cref{app_sec:expsetup}). 

Table~\ref{tab:opt} shows that, 
on all generation tasks,
\mezo{} outperforms zero-shot, ICL, and LP, and achieves comparable performance to FT. 
Considering that 
many applications of fine-tuning LMs -- including instruction tuning or domain adaptation -- target generation tasks, 
our results underscore the potential of \mezo{} as a memory-efficient technique 
to optimize large LMs for realistic and exciting applications.

\paragraph{\mezo{} scales up to 66 billion parameter models.}
We demonstrate the efficacy of \mezo{} on even larger models, up to 66B, in Table~\ref{tab:large}. 
While directly fine-tuning  models at such scales is extremely costly (\Cref{sec:memory}), 
\mezo{} can effectively optimize these models and outperform zero-shot and ICL.

\subsection{Training with non-differentiable objectives}

\label{sec:nondiff}



\begin{table}[t!]
\centering
\resizebox{0.66\textwidth}{!}{
    \begin{tabular}{lccccc}
    \toprule
     Model & \multicolumn{4}{c}{RoBERTa-large (350M)} & OPT-13B \\
     Task  & \tf{SST-2} & \tf{SST-5} & \tf{SNLI} & \tf{TREC}  & \tf{SQuAD} \\
    \midrule
    Zero-shot & 79.0 & 35.5& 50.2 & 32.0  & 46.2\\
    Cross entropy (FT) & 93.9 & 55.9 &  88.7  & 97.3  & 84.2 \\
    Cross entropy (\mezo{}) & 93.3 & 53.2 & 83.0  & 94.3  & 84.7\\
    Accuracy/F1 (\mezo{}) & 92.7 &  48.9 & 82.7  & 68.6 & 78.5 \\
    \bottomrule
    \end{tabular}}
    \vspace{5pt}
    \caption{
        \mezo{} with non-differentiable objectives. For classification ($k=512$), we use \mezo{} with full-parameter and optimize accuracy; for SQuAD (1,000 examples), we use \mezo{} (prefix) and F1. 
    }
    \vspace{-10pt}
    \label{tab:nondiff}
\end{table}
We demonstrate the efficacy of \mezo{} for optimizing non-differentiable objectives
through initial experiments. 
Accuracy and F1 are used as the respective objectives (details in \Cref{sec:detail_nondiff}).
Table~\ref{tab:nondiff} 
reveals that \mezo{} with accuracy/F1 successfully optimizes LMs with superior performance to zero-shot. 
Although minimizing cross entropy results in stronger performance,
these preliminary findings highlight the promising potential of applying \mezo{} to optimize non-differentiable objectives without clear differentiable surrogates,
such as human preferences~\citep{ouyang2022training}.

\subsection{Memory usage and wall-clock time analysis}
\label{sec:memory}


\begin{figure}[t!]
    \centering
    \begin{minipage}[b]{0.55\textwidth}
      \centering
      \includegraphics[width=\textwidth]{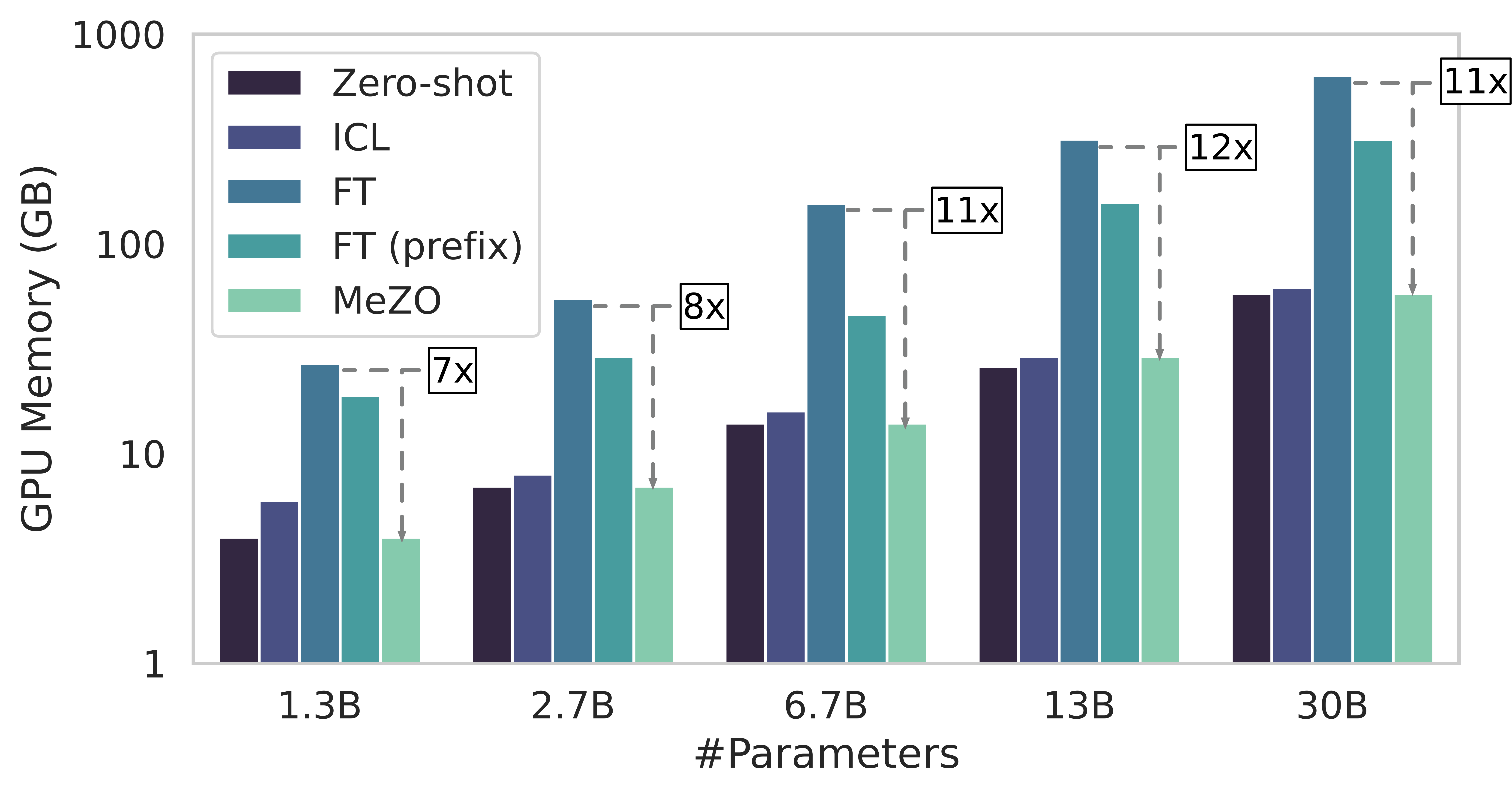}
      \vspace{-16pt}
      \caption{GPU memory consumption with different OPT models and tuning methods on MultiRC (400 tokens per example on average). 
      }
            \vspace{-10pt}
      \label{fig:memory_fig}
    \end{minipage}
    \hfill
    \begin{minipage}[b]{0.43\textwidth}
      \centering
      \resizebox{\textwidth}{!}{
      \begin{tabular}{lccc}
        \toprule
        \multirow{2}{*}{{Hardware}} & \multicolumn{3}{c}{Largest OPT that can fit}\\
        \cmidrule{2-4}
        & \tf{FT} & \tf{FT-prefix} & \tf{Inference} \\
        \midrule
        1$\times$A100 (80GB) & 2.7B& 6.7B & 30B\\
        2$\times$A100 (160GB)& 6.7B& 13B & 66B\\
        4$\times$A100 (320GB)& 13B & 30B & 66B\\
        8$\times$A100 (640GB)& 30B & 66B & 175B$^\dagger$\\
        \bottomrule
        \end{tabular}}
        \vspace{14pt}
      \caption{Largest OPT models that one can tune with specific hardwares and algorithms. $\dagger:$ projected results without actual testing. }
      \vspace{-10pt}
      \label{tab:memory_tab}
    \end{minipage}
  \end{figure}


In this section we profile the memory usage of  
zero-shot, ICL, FT, FT (prefix), and \mezo{}. 
We test OPT models of various sizes with Nvidia A100 GPUs (80GB memory) on MultiRC (average \#tokens=400), and report the peak GPU memory consumption (details in \Cref{sec:profiling}). 


As shown in Figure~\ref{fig:memory_fig} (refer to  \Cref{sec:more_memory} for detailed numbers), \mezo{} exhibits the same memory consumption as zero-shot while offering memory savings of up to $12$ times compared to standard FT and $6$ times compared to FT (prefix).
This advantage enables training larger models within a fixed hardware budget, 
as illustrated in 
Figure~\ref{tab:memory_tab}. 
Specifically, using a single A100 GPU, \mezo{} allows for tuning a model that is $11$ times larger than what is feasible with FT. 


In \Cref{sec:time}, we compare the wall-clock time efficiencies of our implementations of MeZO and Adam fine-tuning. MeZO achieves 7.74$\times$ per-step speedup and requires $8\times$ fewer GPUs with a 30B model, but takes more steps to converge. Overall, MeZO only requires half as many GPU-hours to fine-tune a 30B model compared to full-parameter fine-tuning. The wall-clock benefit of \mezo{} is not inherent to the algorithm and is highly dependent on the implementation. We primarily provide this information as a demonstration that \mezo{} does not take a prohibitively long time to run.

The above measurements are dependent on the computing infrastructure.
In \Cref{app_sec:theory_memory}, we compare the theoretical time-memory tradeoff of \mezo{} and backpropagation and find that \mezo{} is always more memory-efficient than backpropagation 
and is often more time-efficient.
The above analyses also do not consider recent advances (e.g., gradient checkpointing~\citep{chen2016training}, FlashAttention~\citep{dao2022flashattention}, and quantization~\citep{dettmers2022bit}).
We leave investigating the how \mezo{} works with these methods to future work.

%
%
%
\section{Theory}\label{sec:theory}
Our theoretical analysis highlights why \mezo{} can optimize large LMs, although a number of classical results~\citep{nemirovskij1983problem,jamieson2012query,raginsky2011information,agarwal2012information,nesterov2017random} suggest that optimization should be catastrophically slow when training so many parameters.
The inclusion of a simple prompt is crucial for \mezo{} to succeed (\Cref{app_sec:ablations}). 
Past works~\cite{saunshi2021a,malladi2023kernelbased} have suggested that including such a prompt ensures that the fine-tuning objective is closely related to the pre-training one.
As such, here, we make the assumption that the model has already been trained for many steps on the fine-tuning objective, which implies that the loss landscape exhibits favorable conditions (\Cref{assume:low_eff_rank}).
Then, we derive a convergence rate independent of the number of parameters.
We show that the loss decreases per step at a rate independent of the parameter dimension $d$ (\Cref{thm:rate_comparison}), and that, under stronger conditions, the algorithm converges in time independent of $d$ (\Cref{lem:global_ZO-SGD}).
Together, these results imply that \mezo{} is not catastrophically slower than SGD when fine-tuning.\footnote{\Cref{sec:exp} uses the standard choice of Adam for FT; we provide SGD experiments in~\Cref{app_sec:roberta_more}.}
For ease of illustration, we assume that $\vz$ is sampled from a sphere with radius $\sqrt{d}$, and in \Cref{sec:app_gaussian_z}, we derive the rate for a general Gaussian $\vz$, which was used in the experiments.

We follow classical analyses of \sgd{} and replace the minibatch gradient estimate with SPSA (\Cref{def:spsa}). 
Consider the minibatch \sgd{} update $\btheta_{t+1} \leftarrow \btheta_t - \eta \nabla \cL(\btheta; \cB_t)$
where $\cB_t$ is a minibatch drawn uniformly from $\cD^B$. 
Crucially, the SGD minibatch gradient estimate is unbiased.
\begin{definition}[Unbiased Gradient Estimate]
    Any minibatch gradient estimate $\vg(\vtheta, \cB)$ is said to be unbiased if $\E[\vg(\vtheta, \cB)] = \nabla \cL(\btheta)$.
\end{definition}

\subsection{Per-step analysis}
The classical descent lemma uses a Taylor expansion to study how \sgd{} reduces the loss at each optimization step.
It highlights that when the gradient covariance is large, the maximum possible decrease in loss at each optimization step is small, thereby resulting in slower optimization.
\begin{lemma}[Descent Lemma]\label{lem:SGD_descent}
 Let $\cL(\btheta)$ be $\ell$-smooth.\footnote{This is satisfied for the standard cross-entropy objective.} For any unbiased gradient estimate $\vg(\vtheta, \cB)$,
 \begin{align}\label{eq:descent_sgd}
     \E[\cL(\btheta_{t+1}) \mid \vtheta_t] - \cL(\btheta_t)
     &\leq  - \eta\norm{\nabla \cL(\btheta_t)}^2 + \frac12\eta^2\ell\cdot \E[\norm{\vg(\vtheta, \cB_t)}^2].
 \end{align}
\end{lemma}
The descent lemma highlights the importance of the gradient norm, which we derive for \mezo{} below.
\begin{lemma}\label{lem:covariance}
Let $\cB$ be a random minibatch of size $B$. Then, the gradient norm of \mezo{} is
    \begin{align*}
    \E_x\left[\norm{\hat \nabla \cL(\btheta; \cB)}^2\right] = \frac{d + n - 1}{n}\E\left[\norm{\nabla\cL(\vtheta;\cB)}^2\right].
    \end{align*}
where $n$ is the number of $\vz$  sampled in $n$-SPSA (\Cref{def:spsa}) and $d$ is the number of parameters.
\end{lemma}
Thus, in the usual case where $n\ll d$, \mezo{} has a much larger gradient norm than SGD.\footnote{All of our experiments use $n=1$.}
The descent lemma also shows that to guarantee loss decrease, one needs to choose the learning rate as
\begin{align}
\eta \le \frac{2\norm{\nabla \cL(\btheta_t)}^2}{\ell \cdot \E[\norm{\vg(\vtheta, \cB)}^2]} \qquad\xRightarrow{\text{\normalsize\Cref{lem:covariance}}} \qquad \eta_{\text{ZO}} = \frac{n}{d + n - 1}\eta_{\text{SGD}}
\label{eq:lrs}
\end{align}
where $\eta_{\text{ZO}}$ and $\eta_{\text{SGD}}$ are the maximum permissible learning rates for \mezo{} and \sgd{} respectively. Thus we see that without any further assumptions, \mezo{} can slow optimization by decreasing
the largest permissible learning rate 
by a factor of $d$. 
Moreover, \mezo{} reduces the loss decrease that can be obtained at each step and, as a consequence, slows convergence by a factor of $d$ as well.

Surprisingly, our experiments show that \mezo{} can quickly optimize pre-trained models with billions of parameters, and reducing the number of tuned parameters via PEFT techniques does not substantially accelerate optimization (\Cref{app_sec:conv_zo_full_peft}).
We attribute these phenomena to the Hessian of the loss exhibiting small local effective rank.
It is prohibitively expensive to directly measure the effective rank of the Hessian of a large LM on a reasonably sized dataset. 
However, many previous works have shown that the Hessian of the loss for deep neural networks trained by SGD has remarkably low effective rank \citep{papyan2018full,papyan2020traces,ghorbani2019investigation,yao2020pyhessian, wu2020dissecting, sagun2017empirical}. 
In particular, the bulk of the spectrum concentrates around $0$ with only a small number of outliers, and the number of these outliers is an upper bound on the effective rank. 
In addition, prior works \citep{aghajanyan2021intrinsic,li2018measuring} have demonstrated that LM fine-tuning can occur in a very low dimensional subspace ($<200$ parameters), which further supports the below assumption.
We formalize the assumption on the effective rank below.
In particular, we require an upper bound on the Hessian in a neighborhood around the current iterate to have effective rank at most $r$. 
\begin{assumption}[Local $r$-effective rank]\label{assume:low_eff_rank}
Let $G(\btheta_t) = \max_{(\bx, \by) \in \cD}\norm{\nabla\cL(\btheta_t;\{(\bx, \by)\})}$. There exists a matrix $\bH(\btheta_t) \preceq \ell \cdot \bI_d$ such that:
\begin{enumerate}
    \item For all $\vtheta$ such that $\norm{\vtheta - \vtheta_t} \le \eta d G(\vtheta_t)$, we have $\nabla^2 \cL(\btheta) \preceq \bH(\btheta_t)$.
    \item The effective rank of $\bH(\btheta_t)$, i.e $\tr(\bH(\btheta_t))/\norm{\bH(\btheta_t)}_{op}$, is at most $r$. 
\end{enumerate}
\end{assumption}


Under this assumption, we show that the convergence rate of \zosgd{} does not depend on the number of parameters. 
Instead, the slowdown factor only depends on the effective rank of the Hessian.
\begin{theorem}[Dimension-Free Rate]\label{thm:rate_comparison}
Assume the loss exhibits local $r$-effective rank (\Cref{assume:low_eff_rank}). If $\btheta_{t+1} = \btheta_t - \etazo \hat \nabla \cL(\btheta_t; \cB)$ is a single step of \zosgd{} using the $n$-SPSA estimate with a minibatch of size $B$, then there exists a $\gamma = \Theta(r/n)$ such that the expected loss decrease can be bounded as
\begin{equation}
    \E[\cL(\btheta_{t+1}) \mid \btheta_t] - \cL(\btheta_t) \le - \etazo\norm{\nabla \cL(\btheta_t)}^2 + \frac12\etazo^2\ell\cdot \gamma \cdot\E[\norm{\nabla\cL(\vtheta;\cB)}^2]
\end{equation}
\end{theorem}

By applying \Cref{eq:lrs}, we can directly compare to the \sgd{} descent lemma.
\begin{corollary}\label{cor:rate_comparison}
Choosing the learning rate $\etazo = \gamma^{-1}\cdot\etasgd$, \zosgd{} obtains a loss decrease of
\begin{equation}
    \E[\cL(\btheta_{t+1}) \mid \btheta_t] - \cL(\btheta_t) \le \frac{1}{\gamma}\cdot\left[-\etasgd\norm{\nabla \cL(\btheta_t)}^2 + \frac12\etasgd^2\ell\cdot\E[\norm{\nabla\cL(\vtheta;\cB)}^2] \right].
    \label{eq:ZO_SGD_descent}
\end{equation}
\end{corollary}

Here we see that comparing to SGD, the slowdown factor of \zosgd{} scales with the local effective rank $r$, which we argue is much smaller than the number of parameters $d$.
The above analysis focuses on how much \zosgd{} and \sgd{} decrease the loss at each step.
Below, we show that under stronger assumptions about the loss landscape, we can obtain rates for how quickly the \zosgd{} algorithm converges to an optimal value.

\subsection{Global convergence analysis}
We show that the global convergence rate also slows by a factor proportional to the local effective rank under stronger assumptions about the loss landscape.
We assume that the landscape obeys the classical PL inequality: the gradient norm grows quadratically with the suboptimality of the iterate.
\begin{definition}[PL Inequality]
    Let $\cL^* = \min_{\btheta} \cL(\btheta)$. The loss $\cL$ is $\mu$-PL if, for all $\btheta$, $
        \frac12\norm{\nabla \cL(\btheta)}^2 \ge \mu(\cL(\btheta) - \cL^*).
    $
\end{definition}
The PL inequality is not as strong as assuming that optimization exhibits kernel-like dynamics, but it ensures that the landscape is amenable to analysis \citep{karimi2020linear}. 
In addition to the PL inequality, we assume the trace of the gradient covariance is bounded, so noise does not disrupt the trajectory too drastically.
\begin{definition}[Gradient Covariance] The SGD gradient estimate on a minibatch of size $B$ has covariance $
    \bSigma(\vtheta) = B(\E\left[\nabla\cL(\vtheta; \cB)\nabla\cL(\vtheta;\cB)^\top\right] - \nabla \cL(\btheta)\nabla\cL(\btheta)^\top)$.
\end{definition}
As we show in~\Cref{sec:global_proofs}, this assumption holds for common loss functions such as square loss or binary cross entropy for several settings (e.g., kernel behavior \citep{malladi2023kernelbased}).
With these two assumptions, we show that \zosgd{} has a slowdown proportional to the effective rank $r$, not the parameter dimension.
\begin{lemma}[Global Convergence of \zosgd{}]\label{lem:global_ZO-SGD}
    Let $\cL(\btheta)$ be $\mu$-PL and let there exist $\alpha$ such that $\tr(\bSigma(\btheta)) \le \alpha (\cL(\btheta) - \cL^*)$ for all $\btheta$. Then after
    \begin{align*}
        t = \cO \left(\left(\frac{r}{n} + 
        1 \right)\cdot\underbrace{\left(\frac{\ell}{\mu} + \frac{\ell\alpha}{\mu^2B}\right)\log\frac{\cL(\vtheta_0) - \cL^*}{\epsilon}}_{\text{\normalsize SGD rate (\Cref{lem:global_SGD})}}\right)
    \end{align*}
    iterations of \zosgd{} we have $\E[\cL(\btheta_t)] \le \cL^* + \epsilon.$
\end{lemma}

\section{Related work}
\paragraph{Zeroth-order optimization}
Many classical lower bounds have been derived for ZO-SGD in the  
 strongly convex and convex settings~\citep{jamieson2012query,agarwal2012information,raginsky2011information,duchi2015optimal,shamir2017optimal,nemirovskij1983problem} as well as non-convex~\citep{wang2020zeroth}.
These bounds generally depended on the number of parameters $d$.
More recently, \citep{wang2018stochastic,balasubramanian2018zeroth,cai2022zoro} showed that if the gradient has low-dimensional structure,
then the query complexity scales linearly with the intrinsic dimension and logarithmically with the number of parameters,
though the estimation has at least $\Omega(sd\log d)$ memory cost. 
Additional tricks such as
sampling schedules~\citep{bollapragada2018adaptive} and other variance reduction methods~\citep{ji2019improved,liu2018zeroth} can be added to \zosgd{}. 
ZO has inspired distributed methods~\citep{tang2019distributed,hajinezhad2018gradient} and 
black-box adversarial example generation~\citep{cai2021zerothorder,liu2018signsgd,chen2017zoo,liu2020primer} in deep learning.
\citet{ye2018hessian,balasubramanian2022zeroth} estimate the Hessian to perform ZO optimization along important directions.
There are also ZO methods that optimize without estimating the gradient~\citep{golovin2020gradientless,mania2018simple,hinton2022forwardforward}.

\paragraph{Memory-efficient backpropagation}
\label{sec:memory_efficient_backprop}
Several algorithms have been proposed to efficiently approximate backpropagation by sparsifying gradients~\citep{sun17meprop,wei2017minimal}, approximating Jacobians~\citep{abdel2008low,choromanski2017blackbox}, and subsampling the computational graph~\citep{oktay2020randomized,adelman2021faster}. 
However, these methods may accrue large approximation errors for deep networks. 
Gradient checkpointing~\citep{chen2016training}
reduces memory cost of backpropagation at the cost of recomputing 
some activations. 
FlashAttention~\citep{dao2022flashattention}
also reduces memory cost by
recomputing attention matrices. 
\citet{dettmers2022gptint,dettmers2022bit}
explore quantization of large LMs' weights and optimizer states,
which leads to memory reduction in both training and inference.


\paragraph{Gradient-free adaptation of large language models}
BBT and BBTv2~\citep{sun2022black,sun2022bbtv2} use evolutionary algorithms to achieve gradient-free optimization; however, due to its sensitivity to high dimensionality, BBT is limited to only optimize a low-dimension projection of prefixes and they focus on RoBERTa-large size models and few-shot settings. 
Other works in ``black-box tuning'' of LMs focus on 
optimizing discrete prompts 
without updating the model,
either via reinforcement learning
\citep{chai2022clip,deng2022rlprompt,diao2022black},
ensemble \citep{hou2022promptboosting},
or iterative search
\citep{prasad2022grips}. 
Concurrent work in~\cite{yang2023iterative} uses iterative forward passes to improve in-context learning performance.

%

\section{Conclusion}
We have shown that \mezo{} 
can effectively optimize large LMs across many tasks and scales. 
Further experiments suggest that \mezo{} can optimize non-differentiable objectives, which backpropagation usually cannot do.
Our theory illustrates why \mezo{} is not catastrophically slow when tuning  billions of parameters.
As a limitation, \mezo{} takes many steps in order to achieve strong performance, though we show that the per-step speedup in \mezo{} can often make fine-tuning with \mezo{} run faster than a standard implementation of fine-tuning with backpropagation.
We did not explore combining \mezo{} with other memory-efficient methods, 
such as FlashAttention~\cite{dao2022flashattention} and quantization~\cite{dettmers2022gptint}, though we hope to investigate this in the future.

We are excited to explore the applicability of \mezo{} to a number of areas, including but not limited to: pruning, distillation, saliency, interpretability, and dataset selection for fine-tuning.
Non-differentiable objectives are a particularly exciting area, given recent advances in tuning large LMs to adapt to human feedback. 
Conducting theoretical analyses for how these efficient gradient estimates impact the performance of different applications is also of interest.

\section*{Acknowledgements}
We thank Xinyi Chen, Yin Tat Lee, Kaifeng Lyu, Tengyu Ma, Abhishek Panigrahi, Nikunj Saunshi, and Mengzhou Xia for their helpful feedback. 
SA and SM are funded by NSR, ONR, SRC, and Simons Foundation.
JDL, AD, and EN acknowledge the support of the ARO under MURI Award W911NF-11-1-0304,  the Sloan Research Fellowship, NSF CCF 2002272, NSF IIS 2107304,  NSF CIF 2212262, ONR Young Investigator Award, and NSF CAREER Award 2144994.
TG is supported by an IBM PhD Fellowship. 
This work is also partially funded by the National Science Foundation (IIS-2211779).

\newpage

\bibliography{bibliography}
\bibliographystyle{plainnat}

\clearpage
\appendix
\section{Algorithmic Ablations}\label{app_sec:ablations}
We perform a number of ablations to select the best algorithm.
As is standard in ZO literature, we consider the main computational cost to be the number of forward passes.
In our case, the number of forward passes can be affected by the number of gradient steps taken, any usage of gradient accumulation, and using more noise samples to reduce the variance of the gradient estimate.

We observed that the performance of \mezo{} improves monotonically with the number of steps, and there does not appear to be any overfitting. 
Hence, when performing algorithmic ablations, we can focus on the efficiency of different algorithms without considering implicit bias. 
This is also reflected in our theoretical analysis.
To ease the computational load, we fix the number of forward passes to $10,000$ and compare many different algorithms for RoBERTa-large on a smaller set of tasks that span sentiment analysis, entailment, and topic classification: SST-2, SNLI, and TREC.
We emphasize that $10,000$ is a small budget and is only used as a setting to compare these ZO algorithms to each other.
We find that using a linearly decreasing learning rate schedule during training, as was done for fine-tuning with backpropagation in~\citep{liu2019roberta}, does not help or hurt \mezo{}. Similarly, using a learning rate warmup leads to identical results on these three tasks. For simplicity, we use a constant learning rate schedule with no warmup for all of the below experiments.
We perform few-shot experiments with $k=16$ and average the results across 5 seeds.

\begin{table*}[h]
\centering
\resizebox{0.65\columnwidth}{!}{
\begin{tabular}{lrc}
\toprule
Experiment & Hyperparameters & Values \\
\midrule
\mezo{} & Batch size & $\{ 16, 64 \}$ $\times$ \\
& Learning rate & $\{1\mathrm{e}{-5}, 1\mathrm{e}{-6}, 1\mathrm{e}{-7} \}$ $\times$  \\
& $\epsilon$ & $\{1\mathrm{e}{-3}, 1\mathrm{e}{-5} \}$ $\times$ \\
& Weight Decay & $\{0, 0.1\}$ \\
\bottomrule
\end{tabular}
}
\caption{The hyperparameter grid used in our ablation experiments. For simplicity, we use a constant learning rate schedule.}
\label{tab:ablation_hyperparameters}
\end{table*}

\subsection{Prompting}\label{app_sec:prompt}
We study if adding a prompt is crucial to the ability of \mezo{} to optimize the network.  
We use prompts from~\citet{gao-etal-2021-making}.
\citet{malladi2023kernelbased} claimed the prompt makes the optimization trajectory well-behaved, though we note that the current paper considers RoBERTa-large and large autoregressive models while the previous work only studied RoBERTa-base.
We note the similarity between kernel behavior and our theoretical setting in \Cref{sec:theory}.
\mezo{} succeeds on tasks that are reported to not exhibit kernel behavior in \citet{malladi2023kernelbased}, so we investigate whether or not the prompt is necessary.

\begin{table*}[h]
\centering
\begin{tabular}{lccc}
\toprule
 & SST-2 & SNLI & TREC \\
\midrule
Prompt & 89.6 (1.2) & 65.1 (6.2) & 66.7 (6.2) \\
No Prompt & 51.9 (2.9) & 34.8 (2.1) & 19.5 (9.0) \\
\bottomrule
\end{tabular}
\caption{Experiments using \mezo{} to fine-tune models with and without a prompt.}
\label{tab:prompt_vs_no_prompt}
\end{table*}

Both experiments followed the grid in \Cref{tab:ablation_hyperparameters}, but we also expanded the grid to include a learning rate of $1\mathrm{e}-4$ for the no prompt case.
As a result of these experiments, we fix the setting to prompt-based fine-tuning for all of the below experiments. 
\subsection{Sample schedules}\label{app_sec:samp_sched}
One can sample $n_t$ noise vectors at the $t$th step and use $n_t$-SPSA to compute the gradient estimate.
Similar ideas were proposed in~\citet{bollapragada2018adaptive,cai2022zoro}.
We study the effect of linearly increasing and constant sampling schedules in the ablation setting. 
The intuition for the linearly increasing schedule is that the optimizer may need a higher fidelity gradient as it approaches the minimum.
Increasing the number of $z$s can speed up optimization by reducing the gradient variance, but doing so also increases the number of forward passes required for each optimization step, so there is a trade-off to study.
We note that increasing the number of $z$s should be accompanied by a proportional scaling of the learning rate, analogous to the linear scaling rule proposed in~\citep{goyal2017accurate} (theoretical justification can follow the SDE technique~\citep{li2021on}).
\Cref{tab:samp_sched} shows no consistent benefit in one schedule over the other, and it demonstrates that increasing the $n$ in $n$-SPSA while fixing the number of forward passes allowed results in only marginal gains at best.

\begin{table*}[h]
\centering
\begin{tabular}{llccc}
\toprule
$n$ & Schedule & SST-2 & SNLI & TREC \\
\midrule
$n=1$ & Constant & 89.6 (1.2) & 65.1 (6.2) & \textbf{66.7 (6.2)} \\ 
$n=4$ & Constant & 89.5 (1.1) & \textbf{68.6 (3.2)} & 62.3 (5.6) \\
$n=4$ & Linear & 89.6 (1.4) & 65.3 (6.4) & 66.1 (5.5)\\
$n=16$ & Constant & \textbf{90.4 (0.7)} & 67.0 (3.4) & 62.8 (6.3) \\
$n=16$ & Linear & 88.9 (1.2) & 62.8 (5.9) & 64.2 (5.3) \\
\bottomrule
\end{tabular}
\caption{Experiments using \mezo{} with different schedules for $n$. We scale the learning rate proportionally to the number of $\vz$'s sampled.}
\label{tab:samp_sched}
\end{table*}

\section{\mezo{} Variants}\label{sec:alg_variants}
There is a rich history of transferring ideas from first order optimization to enhance ZO algorithms.
Below, we highlight several variants of \mezo{} that did not achieve as high performance as the algorithm presented in Algorithm~\ref{alg:zo_sgd}.

\subsection{Memory-efficient n-SPSA}
We highlight how \mezo{} can perform $n$-SPSA (\Cref{def:spsa}) efficiently for $n>1$ in Algorithm~\ref{alg:zo_sgd_multin}.
In particular, if sampling $n$ $\vz$ vectors and averaging the projected gradients, we require storing $2n$ additional scalars: the random seeds and the projected gradients.
The same caveat about perturbing individual weights versus entire weight matrices still applies here (see~\Cref{sec:prelims}).

\begin{figure}[!htbp] 
\centering
\begin{algorithm}[H]
  \SetKwFunction{perturb}{PerturbParameters}
  \SetKwProg{sub}{Subroutine}{}{}
  \SetKwComment{Comment}{$\triangleright$\ }{}
  
  \textbf{Require}: parameters $\vtheta\in\RR^d$, 
  loss $\cL:\RR^d\to\RR$, 
  step budget $T$,
  perturbation scale $\epsilon$, 
  batch size $B$,
  learning rate schedule $\{\eta_t\}$, 
  $n$ for $n$-SPSA estimate (\Cref{def:spsa}) \\
  \vspace{0.2cm}
  \For{$t=1,...,T$} { 
  	$\texttt{seeds, projected\_grads} \gets\texttt{[]}$ \Comment*[f]{Will each contain $n$ scalars} \\
    \For{$j=1,...,n$} {
	    Sample batch $\cB\subset \cD^B$ and random seed $s$ \\
	    $\vtheta\gets$ \perturb{$\vtheta, \epsilon, s$} \\
	    $\ell_+\gets\cL(\vtheta;\cB)$ \\
	    $\vtheta\gets$ \perturb{$\vtheta, -2\epsilon, s$} \\
	    $\ell_-\gets\cL(\vtheta;\cB)$ \\
	    $\vtheta\gets$ \perturb{$\vtheta, \epsilon, s$} \Comment*[f]{Reset parameters} \\
	    $\texttt{projected\_grad}\gets (\ell_+ - \ell_-) / (2\epsilon)$ \\
	    $\texttt{projected\_grads[j]}\gets \texttt{projected\_grad}$ \\
	    $\texttt{seeds[j]}\gets s$ \\
	    }
	    \BlankLine
	    \For{$j=1,...,n$} {
	    	Reset random number generator with seed $\texttt{seeds[j]}$ \\ 
		    \For{$\theta_i\in\vtheta$}{
		        $z\sim\cN(0,1)$ \\
		        $\theta_i\gets\theta_i - (\eta_t/n) * \texttt{projected\_grads[j]} * z$  \Comment*[f]{Avg grad for $\vz_1,...,\vz_n$}
		 }
	    
    }
  } 
  
  \vspace{0.2cm}
  \sub{\perturb{$\vtheta$, $\epsilon$, $s$}}{
  Reset random number generator with seed $s$  \Comment*[f]{For sampling $z$} \\
    \For{$\theta_i\in\vtheta$ }{  
        $z\sim\cN(0,1)$ \\
        $\theta_i\gets\theta_i+\epsilon z$ \Comment*[f]{Modify parameters in place} 
    }
    \Return $\vtheta$
  }
  
  \caption{\mezo{} with $n>1$}

  \label{alg:zo_sgd_multin}
\end{algorithm}
\end{figure}

\subsection{Augmenting \mezo{} with Gradient History}\label{app_sec:adam_zo}
The $n$-SPSA algorithm merely provides a gradient estimate that can subsequently be used in place of the gradient in any gradient-based optimizer. 
Many popular optimizers, such as Adam and SGD with momentum, require storing some historical information about gradients (e.g., a moving average).
This requirement causes such algorithms to require $2\times$ or $3\times$ the memory that is needed for SGD.

However, one advantage of \mezo{} is that the gradient history can be recomputed at each step without requiring much additional memory.
In reference to Algorithm~\ref{alg:zo_sgd}, note that the gradient only needs \texttt{projected\_grad} and the random seed $s$ used to compute the perturbation $\vz$, so we need to only store $2$ scalars per step to reproduce the gradient history (i.e., up to $2T$ scalars during training).
This is a substantial reduction in added memory overhead that is usually needed for using Adam or momentum instead of vanilla SGD.

\Cref{tab:roberta} illustrates that \mezo{}-Adam can sometimes improve the performance of \mezo{}, though each gradient step requires additional computation (but no additional forward passes).
We leave it to future work to investigate when \mezo{}-Adam may be more useful than \mezo{}.

\begin{table*}[h]
\centering
\resizebox{0.75\columnwidth}{!}{
\begin{tabular}{lrc}
\toprule
Experiment & Hyperparameters & Values \\
\midrule
\mezo{}-Adam & Batch size & $64$ \\
& Learning rate & $\{1\mathrm{e}{-6}, 1\mathrm{e}{-5}, 1\mathrm{e}{-4}, 5\mathrm{e}{-4}, 1\mathrm{e}{-3} \}$ \\
& $\epsilon$ & $1\mathrm{e}{-3}$ \\
& Weight Decay & $0$ \\
\bottomrule
\end{tabular}
}
\caption{The hyperparameter grid used for \mezo{}-Adam. For simplicity, we use a constant learning rate schedule. }
\label{tab:adam_hyperparameters}
\end{table*}

\subsection{Modifying the Variance of \mezo{}}
Our theory in~\Cref{sec:theory} sketches the well-known fact that the variance of the stochastic gradient estimate can impact the rate of optimization.
ZO methods can be combined with standard variance reduction techniques to possibly improve optimization speed.
For example, \citet{liu2018zeroth} designed a variance reduced ZO algorithm, analogous to SVRG~\citep{johnson2013accelerating}, to improve the speed of convergence.
Below, we show that several variance reduction methods (e.g., using the gradient norm) can be implemented in a memory-efficient manner. 
However, when controlling for the total budget of forward passes (i.e., function queries), these methods are not as performant as \mezo{}.
We nevertheless present them to demonstrate the ease with which \mezo{} can be adapted, and we suggest these methods may be useful for optimizing more complex objectives.

First, we define a general SPSA estimate that has the same expectation (i.e., the true gradient) but has a scaled variance. 
\begin{definition}[Variance-Modified SPSA]
	Given a matrix $D = \diag(\vd)$, the variance modified SPSA computes
	\begin{align*}
		\tilde\nabla\cL(\vtheta;\cB) &= \frac{\cL(\vtheta + \epsilon (\vd^{-1}\odot\vz); \cB) - \cL(\vtheta - \epsilon (\vd^{-1}\odot\vz); \cB)}{2\epsilon}(\vd\odot\vz)
	\end{align*}
	where $\vd\in\RR^d$ has nonzero entries and $\vd^{-1}$ denotes the coordinatewise reciprocal.
	\label{def:variance_modified_spsa}
\end{definition}
The above SPSA variant is an unbiased estimator of the gradient, because $\E[\tilde\nabla\cL(\vtheta;\cB)] = \E[D^{-1} \vz\vz^\top D\nabla\cL(\vtheta;\cB)] = \E[\nabla\cL(\vtheta;\cB)]$.
We will draw inspiration from classical methods (i.e., ``control variates'') and choose $\vd$ to be a block vector with gradient norms or parameter norms~\citep{wang2013variance}. 
To select the parameter groups, we split the model by layer, keeping the embedding and the head separate (i.e., RoBERTa-large has $24+2=26$ parameter groups).
It is straightforward to measure the parameter norms without consuming additional memory. 
We can measure the gradient norms without performing backpropagation, as shown below.
\begin{proposition}[ZO Estimate of Gradient Norm of $\ell$th Layer]
	Define $\vz_\ell$ to have $z\sim\cN(0,1)$ in each coordinate corresponding to parameters in the $\ell$th layer and $0$ everywhere else. Then, we can estimate the norm of the gradient of the loss w.r.t. the $\ell$th layer $\nabla_{\vtheta_\ell}$ as
	$$ \left\|\nabla_{\vtheta_\ell}\cL(\vtheta;\cB)\right\|_2 \approx \left| \frac{\cL(\vtheta + \epsilon\vz_\ell;\cB) - \cL(\vtheta - \epsilon\vz_\ell;\cB)}{2\epsilon} \right| $$
	\label{def:zo_norm_estimate}
\end{proposition}

As is true for SPSA, increasing the number of $\vz_\ell$'s sampled for each value of $\ell$ and averaging the result reduces the variance of the estimate.
The rationale for this estimate is that for any vector $\vv$, $\E_\vz [ ( \langle \vv, \vz \rangle )^2 ] = \|\vv\|_2^2$ for Gaussian $\vz$.
It is clear that this estimate can be computed in a memory efficient way, although it requires $2L$ forward passes to compute gradient norms for $L$ parameter groups.

We show the experimental results for modifying the variance below. We follow the ablation setting and use a fixed budget of $10,000$ steps (\Cref{app_sec:ablations}).
Generally, using the gradient norm to reduce the variance substantially hurts performance (\Cref{tab:grad_variance_modified}).
If we ``cheat'' and allow one backpropagation through the network to estimate the gradient norm, then we see that reducing the variance using the gradient norm does not substantially hurt or help performance.
Modifying the variance using the parameter norm, analogous to layerwise adaptive rate methods, does not substantially impact the performance of \mezo{} (\Cref{tab:param_variance_modified}).

Our observation is that decreasing the variance by setting $\vd$ as the gradient norm does not improve optimization.
This empirical result agrees with the exposition in \Cref{sec:theory} that the straightforward variance analysis (which yields a dependence on the number of parameters $d$) is not the best lens to study the rate of optimization when fine-tuning with \mezo{}.
Our effective rank view in \Cref{thm:rate_comparison} and \Cref{lem:global_ZO-SGD} is likely a better characterization of fine-tuning dynamics.
We leave it to future work to explore if these methods can be useful for other more complex objectives.

\begin{table*}[h]
\centering
\begin{tabular}{ccccc}
\toprule
 Recompute $\vd$ & ZO estimate of $\vd$ & SST-2 & SNLI & TREC \\
\midrule
\multicolumn{2}{c}{Baseline \mezo{} (Algorithm~\ref{alg:zo_sgd})} & \multicolumn{1}{c}{89.6 (1.2)} & \multicolumn{1}{c}{65.1 (6.2)} & \multicolumn{1}{c}{66.7 (6.2)} \\
\xmark & \xmark & 89.7 (0.8) & 65.2 (5.2) & 64.3 (6.4) \\
\xmark & \cmark & 87.0 (2.5) & 49.6 (9.2) & 32.6 (7.7) \\
\cmark & \cmark & 79.0 (10.3) & 48.9 (2.2) & 38.7 (7.5) \\
\bottomrule
\end{tabular}
\caption{Experiments modifying the variance of \mezo{} using $\vd$ as the gradient norm (see \Cref{def:variance_modified_spsa}). We sometimes recompute $\vd$ at the start of each epoch or use \Cref{def:zo_norm_estimate} to estimate $\vd$ without requiring backpropagation.}
\label{tab:grad_variance_modified}
\end{table*}

\begin{table*}[h]
\centering
\begin{tabular}{cccc}
\toprule
 Recompute $\vd$ & SST-2 & SNLI & TREC \\
\midrule
\multicolumn{1}{c}{Baseline \mezo{} (Algorithm~\ref{alg:zo_sgd})} & \multicolumn{1}{c}{89.6 (1.2)} & \multicolumn{1}{c}{65.1 (6.2)} & \multicolumn{1}{c}{66.7 (6.2)} \\
\xmark & 89.2 (2.1) & 65.4 (4.2) & 64.8 (5.6) \\
\cmark & 88.2 (4.7) & 65.2 (4.0) & 64.7 (5.5) \\
\bottomrule
\end{tabular}
\caption{Experiments modifying the variance of \mezo{} using $\vd$ as the parameter norm (see \Cref{def:variance_modified_spsa}). We sometimes recompute $\vd$ at the start of each epoch.}
\label{tab:param_variance_modified}
\end{table*}

\subsection{Modifying the Expectation of \mezo{}}
The above experiments show that modifying the variance of \mezo{} cannot consistently accelerate its convergence.
However, a simple modification of \Cref{def:variance_modified_spsa} allows us to change the expectation of \mezo{} as well.
This can be used to efficiently estimate coordinate-wise normalized gradient-based optimizer updates (e.g., Adam).

\begin{definition}[Expectation-Modified SPSA]
	Given a matrix $D = \diag(\vd)$, the variance modified SPSA computes
	\begin{align*}
		\tilde\nabla\cL(\vtheta;\cB) &= \frac{\cL(\vtheta + \epsilon (\vd^{-1}\odot\vz); \cB) - \cL(\vtheta - \epsilon (\vd^{-1}\odot\vz); \cB)}{2\epsilon}\vz
	\end{align*}
	where $\vd\in\RR^d$.
	\label{def:exp_modified_spsa}
\end{definition}

Now, we see that $\tilde\nabla\cL(\vtheta;\cB) = \E[D^{-1}\vz\vz^\top \nabla\cL(\vtheta;\cB)]$ so the SPSA estimate is no longer an unbiased estimator for $\nabla\cL(\vtheta)$. 
If we choose $\vd$ to be the gradient norm, for example, then SPSA can estimate the normalized gradient.
Concurrent work in~\citet{tang2023zerothorder} gives another ZO estimate of the normalized gradient while assuming access to only rankings of inputs (instead of the noisy function evaluations available in our setting).
We find that estimating the normalized gradient does not perform as well as directly estimating the gradient (\Cref{tab:grad_exp_modified}).
Regardless, we present this algorithm as a way to highlight that any coordinate-wise operation to the gradient can be applied in a memory-efficient manner.

\begin{table*}[h]
\centering
\begin{tabular}{cccc}
\toprule
 Method & SST-2 & SNLI & TREC \\
\midrule
\multicolumn{1}{c}{Baseline \mezo{} (Algorithm~\ref{alg:zo_sgd})} & \multicolumn{1}{c}{89.6 (1.2)} & \multicolumn{1}{c}{65.1 (6.2)} & \multicolumn{1}{c}{66.7 (6.2)} \\
Estimate of normalized gradient (\Cref{def:exp_modified_spsa}) & 88.0 (1.2) & 60.0 (2.4) & 44.0 (14.0) \\
\bottomrule
\end{tabular}
\caption{Experiments modifying the expectation of \mezo{} using $\vd$ as the gradient norm (see \Cref{def:exp_modified_spsa}). We use the ZO estimate of the gradient norm (\Cref{def:zo_norm_estimate}).}
\label{tab:grad_exp_modified}
\end{table*}

\subsection{One-point estimate}
\label{sec:one_point_estimate}
Here, we investigate the efficacy of one-point gradient estimators in place of the two-point SPSA method. 
Using a one-point estimator instead of SPSA can reduce the \mezo{} running time by half. 
Many one-point estimators have been proposed in the past~\citep{flaxman2005online,spall1997one,vakhitov2009algorithm}. 
For simplicity, we focus on one estimator~\citep{zhang2022new} that has a form reminiscent of SPSA but requires only one forward pass to estimate the gradient at each step.

\begin{definition}[One-Point Gradient Estimate, \citet{zhang2022new}]
	For a loss function $\cL$ evaluated on a batch $\cB_t$ with parameters $\vtheta_t$ at step $t$, we can draw random noise $\vz_t\sim\cN(0, I_d)$ and compute the gradient estimate using hyperparameter $\epsilon$ as written below.
	\begin{equation*}
		\hat\nabla\cL(\vtheta_t;\cB_t) = \frac{\cL(\vtheta_t+\epsilon \vz_t; \cB_t) - \cL(\vtheta_{t-1} + \epsilon \vz_{t-1}; \cB_{t-1})}{\epsilon}
	\end{equation*}
\end{definition}

Notably, this one-point gradient estimate uses the loss at the previous iterate instead of evaluating the loss again at the current iterate. 
As such, this estimator requires only one forward pass at each iterate to estimate the gradient.
For well-behaved loss functions and slow-moving optimization, these two formulas are intuitively similar.
However, \Cref{tab:onepoint} finds this estimator to be much less efficient than SPSA when fixing the number of forward passes.

\begin{table*}[!htbp]

\centering
\resizebox{0.98\textwidth}{!}{
    \setlength{\tabcolsep}{0.3cm}
    \begin{tabular}{lccccccc}
    \toprule
       & \multicolumn{1}{c}{Steps} &  \multicolumn{1}{c}{\textbf{SST-2}} & \multicolumn{1}{c}{\textbf{SST-5}} & \multicolumn{1}{c}{\textbf{SNLI}}  & \multicolumn{1}{c}{\textbf{MNLI}} & \multicolumn{1}{c}{\textbf{RTE}} & \multicolumn{1}{c}{\textbf{TREC}} \\
     & & \multicolumn{2}{c}{------ sentiment ------} & \multicolumn{3}{c}{------ natural language inference ------} & \multicolumn{1}{c}{--- topic ---}\\
    \midrule
    SPSA~\citep{spall1992multivariate} & 20K	& \tf{92.8} (0.5) &\tf{51.3} (0.9)&\tf{82.9} (1.0)&\tf{74.4} (0.8)&	\tf{76.7} (1.7)&\tf{92.7} (0.6) \\
    One-point estimate~\citep{zhang2022new} & 20K	 &90.0 (0.4)&44.6 (2.0)&70.1 (1.5)&57.2 (0.9)&	64.1 (1.0)&	57.3 (5.7) \\
    One-point estimate~\citep{zhang2022new} & 40K	 & 91.8 (0.5)&45.9 (1.7)&74.4 (0.8)	& 61.0 (1.0)	& 68.7 (1.2)&73.0 (3.1)\\
    \bottomrule
    \end{tabular}}
    \caption{
        Comparison between SPSA and a one-point estimate~\citet{zhang2022new}. 
        One-point estimate only does one forward pass per step, thus is twice as fast as two-point estimate per step.
        As such, the number of forward passes after 40K steps with the one-point estimate is the same as the number of forward passes with SPSA after 20K steps.
        The results show that two-point estimate is much more effective than one-point estimate.
    }
    \label{tab:onepoint}
\end{table*}
\section{Memory Analysis}\label{app_sec:theory_memory}
The compute-memory tradeoff of backpropagation is complex to analyze.
\citet{griewank2008evaluating} provides a rigorous theoretical treatment of the problem.
We empirically measure the memory consumption of different methods for commonly used large language models, but here we hope to provide a more rigorous comparison of different gradient estimation algorithms, independent of the software used to implement them.
Below, we summarize some key points that may help readers to understand how the \mezo{} compute-memory tradeoff compares to backpropagation.

Given a network, the first step to perform backpropagation is to decompose the model into easily differentiable blocks. 
We note that this decomposition is not unique. 
For each block, one can choose to cache the resulting output during the forward pass (thereby consuming memory) or instead recompute the output when it is needed (thereby consuming compute).
The below proposition, adapted from Rule 21 in \citet{griewank2008evaluating}, captures this tradeoff. 
\begin{proposition}[Time-Memory Tradeoff for Backpropagation, \citet{griewank2008evaluating}]
    Consider a network containing $N$ bits. For any time-memory tradeoff hyperparameter $c=O(1)$, there exists a backpropagation algorithm that runs in time $O(cN)$ and consumes memory proportional to $O(N^{1/c})$.
\end{proposition}
\citet{grimm1996optimal} also gave sharp bounds for the memory-time product.
Note that the popular gradient checkpointing~\citep{chen2016training} method allows one to tune $c$ with limited precision (i.e., one cannot always further split a differentiable block and observe savings). 
Experiments in~\citet{chen2016training} choose $c=2$ to achieve $O(\sqrt{N})$ memory while consuming $O(2N)$ computation.
In the extreme case, gradient checkpointing allows one to use $O(N\log N)$ computation and $O(\log N)$ memory.

\mezo{} always consumes $2N$ compute and $O(1)$ memory, so it is more compute-efficient at the same memory cost as gradient checkpointing.
Our exposition in \Cref{sec:prelims} discusses that we can perturb groups of parameters together to save time while consuming additional memory. 
However, we do not consider that variant here because it is somewhere in the middle of the compute-memory pareto curve, where we cannot reason about what backpropagation will do. 
In particular, \mezo{} can split groups differently than backpropagation can, since \mezo{} does not require that each parameter group is easily differentiable, so it is hard to compare the two algorithms along the entire pareto curve.

We also compare backpropagation for the $c=1$ case (i.e., storing everything during the forward pass). 
When storing everything, backpropagation consumes $O(N)$ time and $O(N)$ memory. Hence, SPSA consumes slightly more time and substantially less memory than backpropagation at this end of the tradeoff.

Unlike gradient checkpointing, \mezo{} computes only an approximation of the gradient.
This approximation is only useful for fine-tuning with a prompt, making it less broadly useful than gradient checkpointing.
There are other methods that approximate the gradient with less memory consumption than gradient checkpointing (see the Related Work section), though it is unclear how the memory consumption of those algorithms compare to \mezo{}.

\section{Forward Auto-Differentiation}\label{app_sec:fwd_ad}
We discuss the merits of using forward auto-differentiation instead of two forward passes to construct a gradient estimate for fine-tuning.
As $\epsilon\to 0$, the SPSA gradient estimate (\Cref{def:spsa}) can be written as $\vz\vz^\top \nabla\cL(\vtheta;\cB)$. 
The term $\vz^\top\nabla\cL(\vtheta;\cB)$ is a Jacobian-vector product (JVP), and it is well-known that this can be computed in parallel with a single forward pass while consuming additional memory equivalent to that of the largest activation in the model.
To fully compute the gradient estimate, one must store $\vz$ on the GPU while performing inference, so we observe that this algorithm requires more memory than MeZO.

We note that implementation of the forward auto-differentiation algorithm is not well-supported in PyTorch at the time of writing.
The autograd JVP function computes the JVP in a memory-inefficient way, as noted in the documentation, and the other available methods to compute a JVP are not sophisticated enough to easily scale to a complex LLM.
Computing the JVP is straightforward when using JAX, so we profile the memory consumption of inference and the JVP for RoBERTa-large when using JAX. 
We use batch size 16 with the MultiRC task. Note that JAX may automatically use rematerialization to avoid out of memory errors so we focus on settings in which the memory utilization remains below 50\%. The resulting memory usage during inference, backpropagation, and forward auto-differentiation are reported in ~\Cref{tab:fwd_ad}.

We see that forward auto-differentiation is substantially more memory efficient than backpropagation but less memory efficient than inference.
Furthermore, forward auto-differentiation selects $\epsilon=0$, which removes potentially beneficial third-and-higher order Taylor expansion terms from the estimate. 

\begin{table*}[h]
\centering
\resizebox{0.95\textwidth}{!}{
    \setlength{\tabcolsep}{0.3cm}
    \begin{tabular}{lccc}
    \toprule
     Task  &  \multicolumn{1}{c}{\textbf{Inference} (and \textbf{\mezo{}})} & \multicolumn{1}{c}{\textbf{Backpropagation}} & \multicolumn{1}{c}{\textbf{Forward Auto-Differentiation}} \\
    \midrule
        Excess Memory (MB)  &  327.50  & 24156.23  & 830.66  \\
    \bottomrule
    \end{tabular}}
    \caption{
        Memory consumption of RoBERTa-large when using batch size 16 with the MultiRC task. The reported memory does not include the cost of storing the model on the GPU, which is required for all three cases.
    }
    \label{tab:fwd_ad}
\end{table*}

\section{Experiment setup}
\label{app_sec:expsetup}

\subsection{Datasets}
\label{app_sec:datasets}
For RoBERTa-large, we consider classification datasets: SST-2~\citep{socher2013recursive_sst-2}, SST-5~\citep{socher2013recursive_sst-2},  TREC~\citep{voorhees2000building_trec}, MNLI~\citep{williams2018broad_mnli}, SNLI~\citep{bowman2015large}, and RTE~\citep{dagan2005pascal_rte1,bar2006second,giampiccolo2007third_rte3,bentivogli2009fifth_rte4}.
We follow~\citet{malladi2023kernelbased} in limiting the test set to $1,000$ examples for fast iteration. For training and validation, we have two settings: $k=16$ and $k=512$, which mean that we have 16 or 512 examples per class for both training and validation.

For OPT experiments, we consider the SuperGLUE dataset collection~\citep{wang2019superglue}, including: 
BoolQ~\citep{clark-etal-2019-boolq}, CB~\citep{de2019commitmentbank}, COPA~\citep{roemmele2011choice}, MultiRC~\citep{khashabi2018looking}, ReCoRD~\citep{zhang2018record}, RTE~\citep{dagan2005pascal_rte1,bar2006second,giampiccolo2007third_rte3,bentivogli2009fifth_rte4}, WiC~\citep{pilehvar-camacho-collados-2019-wic}, and WSC~\citep{levesque2012winograd}.
We also include SST-2~\citep{socher2013recursive_sst-2} and two question answering (QA) datasets, SQuAD~\citep{rajpurkar-etal-2016-squad} and DROP~\citep{dua-etal-2019-drop}.
We randomly sample 1,000 examples for training, 500 examples for validation, and 1,000 examples for testing.

\subsection{Prompts}
\label{app_sec:our_prompt}
Table \ref{tab:dataset_statistics} shows the set of downstream tasks and prompts with which we fine-tune RoBERTa-large, which are adapted from \cite{gao-etal-2021-making}.

\newcommand{\cls}{\ttt{[CLS]}}
\newcommand{\tableindent}{~~}
\newcommand{\mask}{\texttt{[MASK]}}
\newcommand{\firstsent}{\ttt{<}$S_1$\ttt{>}}
\newcommand{\secondsent}{\ttt{<}$S_2$\ttt{>}}
\newcommand{\sent}{\ttt{<}$S_1$\ttt{>}}
\newcommand\sys[1]{\textsc{#1}}
\newcommand\ti[1]{\textit{#1}}
\newcommand\ts[1]{\textsc{#1}}
\newcommand\ttt[1]{\texttt{#1}}

\begin{table*}[h]
\centering
\small
\begin{tabular}{lllll}
\toprule
 \tf{Dataset} & $C$  & \tf{Type} & \tf{Prompt} & \tf{Label words} \\
 \midrule
 SST-2 & 2 & sentiment  cls.& {\sent} It was {\mask} . & \{great, terrible\} \\
 SST-5 & 5  & sentiment cls.& {\sent} It was {\mask} . & \{great, good, okay, bad, terrible\} \\
 TREC & 6   & topic cls. & {\mask} : {\sent} & \{Description, Expression, Entity, \\
 & & & & Human, Location, Number\}\\
  MNLI & 3  &  NLI & {\firstsent} ? {\mask} , {\secondsent}  & \{Yes, Maybe, No\}\\
SNLI & 3   & NLI  & {\firstsent} ? {\mask} , {\secondsent} & \{Yes, Maybe, No\} \\
 RTE & 2  & NLI & {\firstsent} ? {\mask} , {\secondsent} & \{Yes, No\}  \\
\bottomrule
\end{tabular}
\caption{The  prompts of the datasets we used in our RoBERTa-large experiments (\Cref{tab:roberta} and \Cref{fig:rob}). 
The prompts are adapted from \cite{gao-etal-2021-making} and include a template and a set of label words that can fill in the \mask token. {\firstsent} and {\secondsent} refer to the first and the second (if any) input sentence.}
\label{tab:dataset_statistics}
\end{table*}

Table \ref{tab:prompt_opt} demonstrates the prompts we use for OPT. 
Note that in OPT experiments we have three types of tasks: classification, multiple-choice, and question answering.
Prompts are adopted from GPT-3~\citep{brown2020language} and PromptSource with minor changes~\citep{bach2022promptsource}.

\newcommand{\nnn}{\textbackslash{}n}
\newcommand{\lword}[1]{{\color{blue}#1}}
\newcommand{\inp}[1]{\ttt{#1}}

\begin{table*}[h]
\centering
\small
\begin{tabular}{lll}
\toprule
 \tf{Dataset}  & \tf{Type} & \tf{Prompt}\\
 \midrule
 SST-2  &  cls.& {\inp{<text>}} It was \lword{terrible}/\lword{great} \\
 RTE & cls. & \inp{<premise>}\\
 &&Does this mean that "\inp{<hypothesis>}" is true? Yes or No?\\
 && \lword{Yes}/\lword{No} \\
CB & cls. & Suppose \inp{<premise>} Can we infer that "\inp{<hypothesis>}"? Yes, No, or Maybe? \\
&& \lword{Yes}/\lword{No}/\lword{Maybe}\\
BoolQ & cls. & \inp{<passage>} \inp{<question>}? \\
&& \lword{Yes}/\lword{No} \\
WSC & cls. & \inp{<text>}\\
&& In the previous sentence, does the pronoun "\inp{<span2>}" refer to \inp{<span1>}? Yes or No? \\
&& \lword{Yes}/\lword{No}\\
WIC & cls. & Does the word "\inp{<word>}" have the same meaning in these two sentences? Yes, No? \\
&& \inp{<sent1>}\\
&& \inp{<sent2>}\\
&& \lword{Yes}/\lword{No} \\
MultiRC & cls. & \inp{<paragraph>}\\
&& Question: \inp{<question>} \\
&& I found this answer "\inp{<answer}". Is that correct? Yes or No? \\
&& \lword{Yes}/\lword{No} \\
COPA & mch. &  \inp{<premise>} so/because \inp{<candidate>}\\
ReCoRD & mch. & \inp{<passage>}\\
& & \inp{<query>.replace("@placeholder", <candidate>)} \\
SQuAD & QA & Title: \inp{<title>}\\
&& Context: \inp{<context>}\\
&& Question: \inp{<question>}\\
&& Answer: \\
DROP & QA& Passage: \inp{<context>}\\
&& Question: \inp{<question>}\\
&& Answer: \\
\bottomrule
\end{tabular}
\caption{
The prompts of the datasets we used in our OPT experiments.  
There are three types of tasks: classification (cls.), multiple-choice (mch.), and question answering (QA).
Prompts are adopted from GPT-3~\citep{brown2020language} and PromptSource~\citep{bach2022promptsource} with minor changes.
\inp{<text>} represents input from the dataset and \lword{Yes} represents label words. 
For inference on multiple choice tasks, we put in different candidates in the prompt and calculate the average log-likelihood for each candidate, and choose the candidate with the highest score. For inference on QA tasks, we use greedy decoding to generate the answer. 
}
\label{tab:prompt_opt}
\end{table*}

\subsection{Hyperparameters}
\label{app_sec:hyper}
We use the hyperparameters in \Cref{tab:mezo_hyperparameters} for \mezo{} experiments on RoBERTa-large (\Cref{tab:roberta} and \Cref{fig:rob}).
Experiments in \Cref{app_sec:ablations} informed the grid; in particular, the choice of $\epsilon$ seemed to not significantly impact performance, and using a larger batch size consistently yielded faster optimization.
We use the hyperparameters in \Cref{tab:opt_hyper} for \mezo{} experiments on OPT. 

Regarding learning rate scheduling and early stopping, we use linear learning scheduling for all fine-tuning with backpropagation experiments and constant learning rate for all \mezo{} experiments. 
For RoBERTa experiments, we evaluate the model on validation sets every 1/10 of total training steps and save the best validation checkpoint. 
All \ft{} experiments use 1K steps and \mezo{} experiments use $100$K steps.
For OPT experiments, we evaluate the model on validation sets every 1/5 of the total training steps and save the best validation checkpoint. 
All FT experiments train for 5 epochs and all \mezo{} experiments use $20$K steps.
Note that FT experiments mostly converge within 5 epochs but we observe that 
\mezo{} performance can still improve with more training steps.

\begin{table*}[h]
\centering
\small
\begin{tabular}{lrc}
\toprule
Experiment & Hyperparameters & Values \\
\midrule
\mezo{} & Batch size & $64$ \\
& Learning rate & $\{1\mathrm{e}{-7}, 1\mathrm{e}{-6}, 1\mathrm{e}{-5} \}$ \\
& $\epsilon$ & $1\mathrm{e}{-3}$ \\
& Weight Decay & $0$ \\
\midrule
\mezo{} (prefix) & Batch size & $64$ \\
& Learning rate & $\{1\mathrm{e}{-2}, 5\mathrm{e}{-3}, 1\mathrm{e}{-3} \}$ \\
& $\epsilon$ & $1\mathrm{e}{-1}$ \\
& Weight Decay & $0$ \\
& \# prefix tokens &$5$\\
\midrule
\mezo{} (LoRA) & Batch size & $64$ \\
& Learning rate & $\{1\mathrm{e}{-5}, 5\mathrm{e}{-5}, 1\mathrm{e}{-4} \}$ \\
& $\epsilon$ & $1\mathrm{e}{-3}$ \\
& Weight Decay & $0.1$ \\
& $(r, \alpha)$ & $(8, 16)$ \\
\midrule\midrule
FT with Adam & Batch size ($k=16$) & $\{2,4,8\}$ \\
 & Batch size ($k=512$) & $\{8,16,32\}$ \\
& Learning Rates &  $\{1\mathrm{e}{-5}, 3\mathrm{e}{-5}, 5\mathrm{e}{-5} \}$ \\
& Weight Decay & $0$\\
\midrule
FT with SGD & Batch size ($k=16$) & $\{2,4,8\}$ \\
 & Batch size ($k=512$) & $\{8,16,32\}$ \\
& Learning Rates & $\{1\mathrm{e}{-4}, 5\mathrm{e}{-4}, 1\mathrm{e}{-3},5\mathrm{e}{-3},1\mathrm{e}{-2} \}$ \\
& Weight Decay & $0$\\
\midrule 
FT (prefix) & Batch size & $\{8,16,32\}$ \\
& Learning Rates & $\{1\mathrm{e}{-2}, 3\mathrm{e}{-2}, 5\mathrm{e}{-2}\}$\\
& Weight Decay & $0$\\
& \# prefix tokens &$5$\\
\midrule
FT (LoRA) & Batch size & $\{4,8,16\}$ \\
& Learning Rates & $\{1\mathrm{e}{-4}, 3\mathrm{e}{-4}, 5\mathrm{e}{-4}\}$\\ 
& $(r, \alpha)$ & $(8, 16)$ \\
\bottomrule
\end{tabular}
\caption{The hyperparameter grids used for RoBERTa-large experiments. \mezo{} uses a constant learning rate schedule, and FT uses linear scheduling. All \ft{} experiments use 1K steps and \mezo{} experiments use $100$K steps. We check validation performance every 1/10 total training steps.}
\label{tab:mezo_hyperparameters}
\end{table*}

\begin{table*}[h]
    \centering
    \small
    \begin{tabular}{lrc}
    \toprule
    Experiment & Hyperparameters & Values \\
    \midrule
    \mezo{} & Batch size & $16$ \\
    & Learning rate & $\{1\mathrm{e}{-6}, 1\mathrm{e}{-7} \}$ or $\{1\mathrm{e}{-6}, 5\mathrm{e}{-7}, 1\mathrm{e}{-7} \}$ for SQuAD and DROP\\
    & $\epsilon$ & $1\mathrm{e}{-3}$ \\
    \midrule
    \mezo{} (prefix) & Batch size & $16$ \\
    & Learning rate & $\{1\mathrm{e}{-2}, 1\mathrm{e}{-3} \}$  or $\{5\mathrm{e}{-2}, 1\mathrm{e}{-2}, 5\mathrm{e}{-3} \}$ for SQuAD and DROP\\
    & $\epsilon$ & $1\mathrm{e}{-1}$ \\
& \# prefix tokens &$5$\\
    \midrule
    \mezo{} (LoRA) & Batch size & $16$ \\
    & Learning rate & $\{1\mathrm{e}{-4}, 5\mathrm{e}{-5} \}$  or $\{1\mathrm{e}{-4}, 5\mathrm{e}{-5}, 1\mathrm{e}{-5} \}$ for SQuAD and DROP\\
    & $\epsilon$ & $1\mathrm{e}{-2}$ \\
    & $(r, \alpha)$ & $(8, 16)$ \\
    \midrule\midrule
    FT with Adam & Batch size & $8$ \\
    & Learning Rates & $\{1\mathrm{e}{-5}, 5\mathrm{e}{-5}, 8\mathrm{e}{-5} \}$\\
    \bottomrule
    \end{tabular}
    \caption{The hyperparameter grids used for OPT experiments. All weight decay is set to $0$. FT uses 5 epochs and linear scheduled learning rates and \mezo{} uses $20$K steps and constant learning rates. 
    We check validation performance and save the best checkpoint every 1/5 total training steps.
    }
    \label{tab:opt_hyper}
    \end{table*}

\subsection{Modeling and implementation}
\label{app_sec:modeling}

For RoBERTa experiments, we follow~\citep{gao-etal-2021-making} for the prompt-based fine-tuning paradigm for masked language models. Please refer to the original paper for more details.

In OPT experiments, for classification tasks, we train the model similarly to~\citep{gao-etal-2021-making}, i.e., we take the logits corresponding to the label words and apply cross entropy loss on them; 
for multiple choice tasks and generation tasks (QA), we only keep the correct candidate and use teacher forcing to train on the correct examples. We only keep the loss on tokens in the candidate part and exclude the prompt part.

For OPT inference on classification and multiple-choice tasks, we use the model to get the average log-likelihood (by tokens) of all the candidates/label words, and predict the one with the highest average log-likelihood. For generation tasks, we use greedy decoding to generate the answer.

For in-context learning, we use 32 examples in the context. 
We also try filling in as many examples as possible in the context but this does not improve performance and sometimes leads to unstable results. Thus we keep the 32-example results.

For linear probing of classification tasks, 
we take the output feature and use \ttt{scipy} package to train a linear classifier.
For multiple-choice tasks and generation tasks, we found that 
this leads to poor results since the output space is the whole vocabulary;
instead, we do head-tuning, where the whole model is fixed except for the LM projection head. 
We use a batch size of $8$ and a learning rate of $\{1\mathrm{e}{-4}\, 5\mathrm{e}{-4}\}$, and train the head for $5$ epochs.

For experiments on 30B and 66B OPT models, we largely follow the OPT hyperparameters except that we do not evaluate the intermediate validation performance and directly use the last checkpoint for evaluation, due to the high storage cost of intermediate checkpoints of large models.

\subsection{Parameter-efficient fine-tuning}
\label{sec:peft}

Fine-tuning and storing a copy of the large language model for each downstream task is expensive.
Parameter-efficient fine-tuning (PEFT) techniques alleviate this problem:
instead of tuning all model parameters, PEFT only tunes a small number of additional parameters (usually less than 1\%) and can often achieve comparable or better performance~\citep{li-liang-2021-prefix,lester-etal-2021-power,ding2022delta}.
The ZO optimizer is compatible with PEFT methods, since ZO can operate on any subset of the model parameters.
We are interested in the following two common PEFT methods, designed for transformers~\citep{vaswani2017attention}.

\tf{LoRA} \citep{hu2021lora} adds a tunable low-rank delta to a linear layer during fine-tuning. 
Suppose a linear layer performed $\mW\vx+\vb$ during pre-training with $\mW\in\RR^{m\times n}$.
When fine-tuning, LoRA introduces two smaller matrices $\mA\in\RR^{m\times r}$ and $\mB\in\RR^{r\times n}$ such that $r\ll\min(m,n)$.
The linear layer is then computed as 
\begin{equation}
	\left(\mW + \frac{\alpha}{r}\mA\mB\right)\vx + \vb
\end{equation}
where $r$ and $\alpha$ are hyperparameters. 
$\mA$ and $\mB$ are trained on the downstream task while $\mW$ is frozen at its pre-trained value.  
In transformers, this modification to the linear layer is applied to the query and value operations of each attention layer.
Empirically, $r$ can be very small, so the number of trainable parameters during fine-tuning is small. 
We choose $r=8$ and $\alpha=16$.

\tf{Prefix-tuning} \citep{li-liang-2021-prefix}
adds a prefix of $m$ tunable representations at each layer and freezes the rest of the model.
The representations are added as new keys and values and treated as additional context during the attention operation.    
We initialize these tunable representations by randomly sampling tokens from the vocabulary and passing them through the LLM to get their keys and values at different attention layers. 
We found this crucial to make prefix tuning stable with \mezo{}, and this trick additionally boosts the performance of prefix tuning with backpropagation, as shown in \Cref{tab:prefixnoreal}.
We also tried the reparameterization trick in~\citep{li-liang-2021-prefix}, which does not help \mezo{} training.
In our experiments, we find $m=5$ to be sufficient to achieve good performance on most tasks.

We also show that \mezo{} is compatible with parameter-efficient fine-tuning methods, such as prefix tuning and LoRA. 
Surprisingly, the performance of \mezo{} does not improve substantially when tuning much fewer parameters, as one might expect from classical analyses (see~\Cref{sec:theory}). 
Accordingly, our theoretical analysis in \Cref{sec:theory} suggests that the convergence rate of ZO-SGD does not depend on the parameter dimension during fine-tuning.

\begin{table*}[!htbp]

\centering
\resizebox{0.95\textwidth}{!}{
    \setlength{\tabcolsep}{0.3cm}
    \begin{tabular}{lcccccc}
    \toprule
     Task  &  \multicolumn{1}{c}{\textbf{SST-2}} & \multicolumn{1}{c}{\textbf{SST-5}} & \multicolumn{1}{c}{\textbf{SNLI}}  & \multicolumn{1}{c}{\textbf{MNLI}} & \multicolumn{1}{c}{\textbf{RTE}} & \multicolumn{1}{c}{\textbf{TREC}} \\
    Type & \multicolumn{2}{c}{------ sentiment ------} & \multicolumn{3}{c}{------ natural language inference ------} & \multicolumn{1}{c}{--- topic ---}\\
    \midrule
    \ft{} (prefix, random init)	& 90.7 (1.7)	&47.2 (2.0)	&70.7 (2.8)	&62.6 (3.3)	&63.5 (4.4)	&83.4 (4.7) \\
    \ft{} (prefix, real act init)	 & 91.9 (1.0)	&47.7 (1.1)	&77.2 (1.3)	&66.5 (2.5)	&66.6 (2.0)	&85.7 (1.3) \\
    \bottomrule
    \end{tabular}}
    \caption{
        Prefix-tuning ablations. We compare randomly-initialized prefixes and real word activation prefixes. Using real word activations significantly outperforms random initialization.
    }
    \label{tab:prefixnoreal}
\end{table*}

\subsection{Training with non-differentiable objectives}
\label{sec:detail_nondiff}
The experiments maximizing the accuracy of a RoBERTa-large model were all conducted using the same grid as \mezo{} in \Cref{tab:mezo_hyperparameters}. 

For OPT experiments on SQuAD with F1 as objective, we use a batch size of $16$.
For \mezo{}, we use 
a learning rate of $\{1\mathrm{e}{-6}, 5\mathrm{e}{-6}, 1\mathrm{e}{-5}\}$ and $\epsilon=1\mathrm{e}{-3}$.
For \mezo{} (prefix), we use
 a learning rate of $\{1\mathrm{e}{-1}, 5\mathrm{e}{-2}, 1\mathrm{e}{-2}\}$ and $\epsilon=1\mathrm{e}{-1}$.

\subsection{Memory profiling}
\label{sec:profiling}

In memory profiling,
we use standard implementation with Huggingface's \ttt{transformers} \citep{wolf-etal-2020-transformers} package.
We did not turn on any advance memory-saving options, e.g., gradient checkpointing.
We set the per-device batch size as $1$ to
test the minimum hardware requirement to run the model with specific optimization algorithms.
For multi-GPU backpropagation, we use fully sharded data parallel (FSDP) \citep{FairScale2021} provided by PyTorch~\citep{pytorch}.
For multi-GPU \mezo{}, we use \ttt{transformers} multi-GPU inference of large models.
We use Nvidia's \ttt{nvidia-smi} command to  monitor the GPU memory usage.
We call a run ``successful'' if there is no out of memory error from GPUs for at least 100 steps.
We also profile fine-tuning with LoRA, but find its memory usage similar to that of fine-tuning with prefix-tuning. Hence here we only show the analysis with prefix-tuning.
\section{More experiment results}
\label{app_sec:more_exp}

\subsection{RoBERTa-large experiments}
\Cref{tab:roberta} contains the detailed numbers corresponding to~\Cref{fig:rob} and also reports the performance of \mezo{}-Adam.
\label{app_sec:roberta_more}

\begin{table*}[!htbp]

\centering
\resizebox{0.95\textwidth}{!}{
    \setlength{\tabcolsep}{0.3cm}
    \begin{tabular}{lcccccc}
    \toprule
     Task  &  \multicolumn{1}{c}{\textbf{SST-2}} & \multicolumn{1}{c}{\textbf{SST-5}} & \multicolumn{1}{c}{\textbf{SNLI}}  & \multicolumn{1}{c}{\textbf{MNLI}} & \multicolumn{1}{c}{\textbf{RTE}} & \multicolumn{1}{c}{\textbf{TREC}} \\
    Type & \multicolumn{2}{c}{------ sentiment ------} & \multicolumn{3}{c}{------ natural language inference ------} & \multicolumn{1}{c}{--- topic ---}\\
    \midrule
        Zero-shot & 79.0 & 35.5 & 50.2 & 48.8 & 51.4 & 32.0 \\
\midrule

    \multicolumn{7}{c}{Gradient-free methods: $k=16$}\\
    \midrule
    LP & 76.0 (2.8)	& 40.3 (1.9)	&66.0 (2.7)	&56.5 (2.5)	&59.4 (5.3)	&51.3 (5.5) \\
    \mezo{} & 90.5 (1.2)	& 45.5 (2.0)	&68.5 (3.9)	&58.7 (2.5) & 64.0 (3.3) &76.9 (2.7) \\
    \mezo{} (LoRA) & 91.4 (0.9) & 43.0 (1.6) & 69.7 (6.0)&64.0 (2.5) & 64.9 (3.6) &73.1 (6.5)\\
    \mezo{} (prefix) & 90.8 (1.7)	&45.8 (2.0)&	71.6 (2.5)	& 63.4 (1.8)& 65.4 (3.9)&80.3 (3.6) \\
    \mezo{}-Adam  & 90.4 (1.4)	& 45.4 (1.5) & 74.1 (2.7)	&64.3 (0.8)$\dagger$ &	59.2 (11.1)$\dagger$&	78.3 (1.4) \\
        \midrule
    \multicolumn{7}{c}{Gradient-based methods: $k=16$}\\
    \midrule
    \ft{} & 91.9 (1.8)	&47.5 (1.9)&	77.5 (2.6)	&70.0 (2.3)	&66.4 (7.2)	&85.0 (2.5) \\
    \ft{} (LoRA)&91.4 (1.7)	&46.7 (1.1)&	74.9 (4.3)	&67.7 (1.4)&	66.1 (3.5)&	82.7 (4.1) \\
    \ft{} (prefix) & 91.9 (1.0)&47.7 (1.1) & 77.2 (1.3) & 66.5 (2.5)	&66.6 (2.0)	&85.7 (1.3) \\
    \midrule\midrule
    \multicolumn{7}{c}{Gradient-free methods: $k=512$}\\ 
    \midrule
    LP & 91.3 (0.5) & 51.7 (0.5) & 80.9 (1.0) & 71.5 (1.1) & 73.1 (1.5) & 89.4 (0.5) \\
    \mezo{} & 93.3 (0.7) & 53.2 (1.4) & 83.0 (1.0) & 78.3 (0.5) & 78.6 (2.0) & 94.3 (1.3) \\
    \mezo{} (LoRA) & 93.4 (0.4) & 52.4 (0.8) & 84.0 (0.8) & 77.9 (0.6) & 77.6 (1.3) & 95.0 (0.7) \\
    \mezo{} (prefix) & 93.3 (0.1)	& 53.6 (0.5)	&84.8 (1.1) &	79.8 (1.2) &77.2 (0.8) & 94.4 (0.7)  \\
    \mezo{}-Adam & 93.3 (0.6) & 53.9 (0.8) & 85.3 (0.8) & 79.6 (0.4) & 79.2 (1.2)  & 95.1 (0.3)\\
    \midrule
    \multicolumn{7}{c}{Gradient-based methods: $k=512$}\\ 
    \midrule
    \ft{} & 93.9 (0.7)	&55.9 (0.9)	&88.7 (0.8)	&84.4 (0.8)&	82.7 (1.4)	&97.3 (0.2)\\
    \ft{} (LoRA) & 94.2 (0.2)&55.3 (0.7)&	88.3 (0.5)	&83.9 (0.6)&	83.2 (1.3)	&97.0 (0.3)\\
    \ft{} (prefix) & 93.7 (0.3)&54.6 (0.7)&88.3 (0.7)&83.3 (0.5)&82.5 (0.8)&97.4 (0.2)\\

    \bottomrule
    \end{tabular}}
    \caption{
    Experiments on RoBERTa-large (350M parameters). LP: Linear probing; \zosgds{}, \zosgds{} (LoRA), and \zosgds{} (prefix): our memory-efficient \zosgd{} (\Cref{sec:memory_efficient_zo}) with full-parameter tuning, LoRA, and prefix-tuning respectively; FT: fine-tuning with Adam. All reported numbers are averaged accuracy (standard deviation). 
    All experiments use prompts (\Cref{app_sec:our_prompt}).
    ZO outperforms zero-shot and LP by a large margin and approaches FT performance with much less memory cost. 
    }
    \label{tab:roberta}
\end{table*}

\paragraph{LP-\mezo{}}
We also compare \mezo{} to performing linear probing and then subsequently performing fine-tuning via \mezo{}, following the analogous suggestion for fine-tuning in~\citet{kumar2022finetuning}.
We use the \mezo{} grid described in~\Cref{tab:mezo_hyperparameters}.
Note that the linear probing checkpoints used here have early stopping, unlike the ones reported in~\Cref{tab:roberta}.
We heuristically implement early stopping by limiting the number of iterations (from $5000$ to $1000$) and increasing the convergence tolerance (from $1\mathrm{e}{-4}$ to $0.01$) in the $\texttt{scipy}$ solver.
Experiments on a few settings show that LP-\mezo{} can sometimes improve performance without increasing the memory consumption (see~\Cref{tab:lp_mezo}).
However, sometimes, linear probing first can severely hurt performance.

\begin{table*}[h]

\centering
\small
    \setlength{\tabcolsep}{0.3cm}
    \begin{tabular}{lccccc}
    \toprule
     Task  &  \multicolumn{1}{c}{\textbf{SST-2}} & \multicolumn{1}{c}{\textbf{SST-5}} & \multicolumn{1}{c}{\textbf{SNLI}} & \multicolumn{1}{c}{\textbf{TREC}} \\
    \midrule
        Zero-shot & 79.0 & 35.5 & 50.2 & 32.0 \\
    FT & 91.9 (1.8) & 47.5 (1.9) & 77.5 (2.6) & 85.0 (2.5)\\
    \midrule
    \mezo{} & 90.5 (1.2) & \textbf{45.5 (2.0)} & 68.5 (3.9) & \textbf{76.9 (2.7)} \\
    LP-\mezo{} & \textbf{91.4 (1.4)} & 41.9 (3.3) & \textbf{70.7 (3.4)} & 54.0 (4.5) \\

    \bottomrule
    \end{tabular}
    \caption{
       Performing linear probing before fine-tuning with \mezo{}, as suggested previously~\citep{kumar2022finetuning}, can sometimes improve performance without increasing the memory overhead. We use $k=16$ for these experiments.
    }
    \label{tab:lp_mezo}
\end{table*}

\subsection{OPT experiments}
\label{app_sec:opt_more}

\Cref{tab:large_more} 
present the full results of OPT-30B and OPT-66B, with detailed \mezo{}  numbers.

\begin{table}[h]
    \centering
    \resizebox{0.75\textwidth}{!}{
        \begin{tabular}{lcccccc}
        \toprule
         Task  & \tf{SST-2} & \tf{RTE} & \tf{BoolQ} & \tf{WSC} & \tf{WIC} & \tf{SQuAD} \\
        \midrule
        30B zero-shot & 56.7 & 52.0 & 39.1 & 38.5 &	50.2 & 46.5\\
        30B ICL & 81.9 & 66.8 & 66.2 &56.7&	51.3 & 78.0\\
        30B \mezo{} & 90.6	&66.4& 67.2&	63.5&	56.3& 85.2\\
        30B \mezo{} (prefix) & 87.5	&72.6&		73.5&	55.8&	59.1&83.9\\
        \midrule
        66B zero-shot & 57.5 & \tf{67.2} & 66.8&	43.3& 50.6&48.1\\
        66B ICL & 89.3 & 65.3&	62.8&	52.9&54.9&81.3\\
        66B \mezo{} & 91.2&	65.7&		72.7&	63.5&58.9& *\\
        66B \mezo{} (prefix) & 93.6&66.4&73.7&	57.7&	58.6&85.0\\
        \bottomrule
        \end{tabular}}
        \vspace{5pt}
        \caption{
            Experiments on OPT-30B and OPT-66B (with 1,000 examples). *: \mezo{} requires further tuning to successfully optimize.
        }
        \label{tab:large_more}
\end{table}

\subsection{Convergence of \mezo{} with full-parameter and PEFT}
\label{app_sec:conv_zo_full_peft}

We demonstrate the convergence rate of 
\mezo{}, \mezo{} (LoRA) and \mezo{} (prefix) 
on SST-2 and SNLI for the first 5,000 steps in Figures~\ref{fig:peft_plot_convg}.
We see that despite the different number of parameters they optimize, 
\mezo{} demonstrates similar training speed on full parameter and PEFT. This agrees with our theory in \Cref{sec:theory}, which shows that \mezo{}'s optimization speed is independent of the number of parameters.

\begin{figure}[h]
    \centering
    \begin{minipage}[b]{0.45\textwidth}
      \centering
      \includegraphics[width=\textwidth]{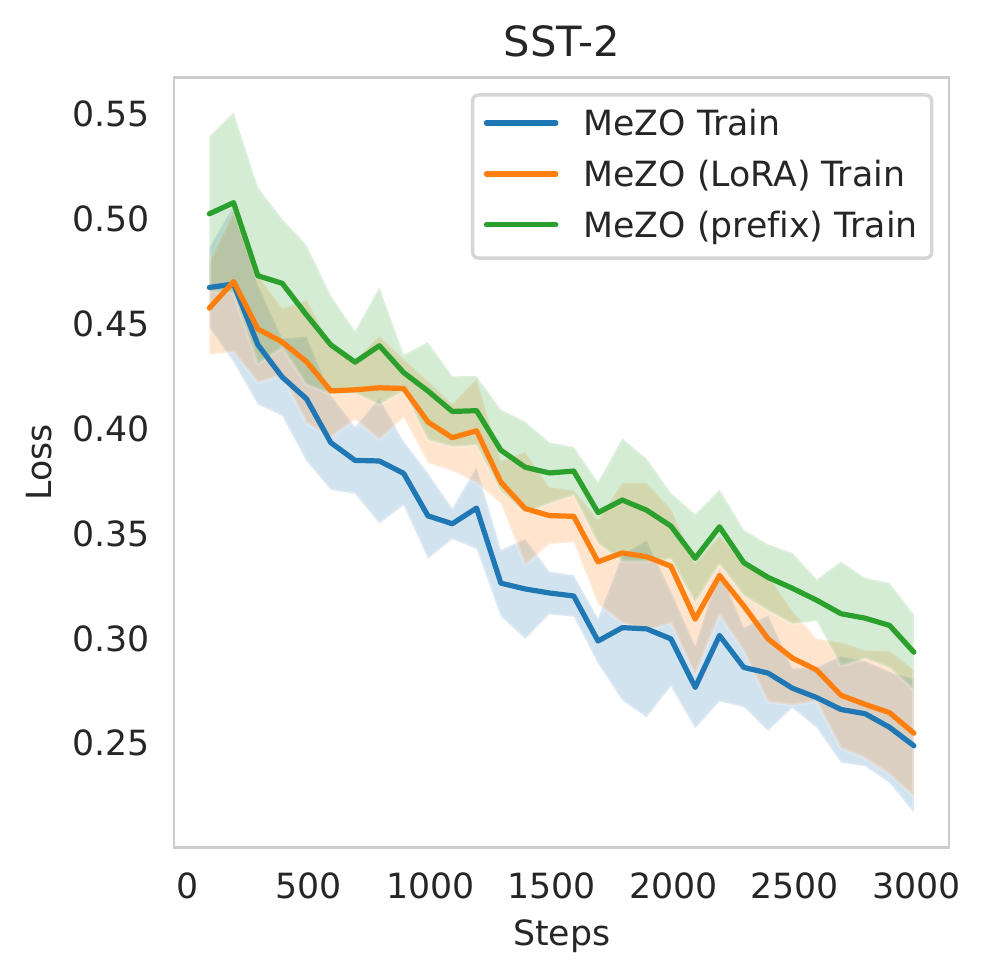}
      \label{fig:peft_plot_sst2}
    \end{minipage}
    \begin{minipage}[b]{0.45\textwidth}
      \centering
        \includegraphics[width=\textwidth]{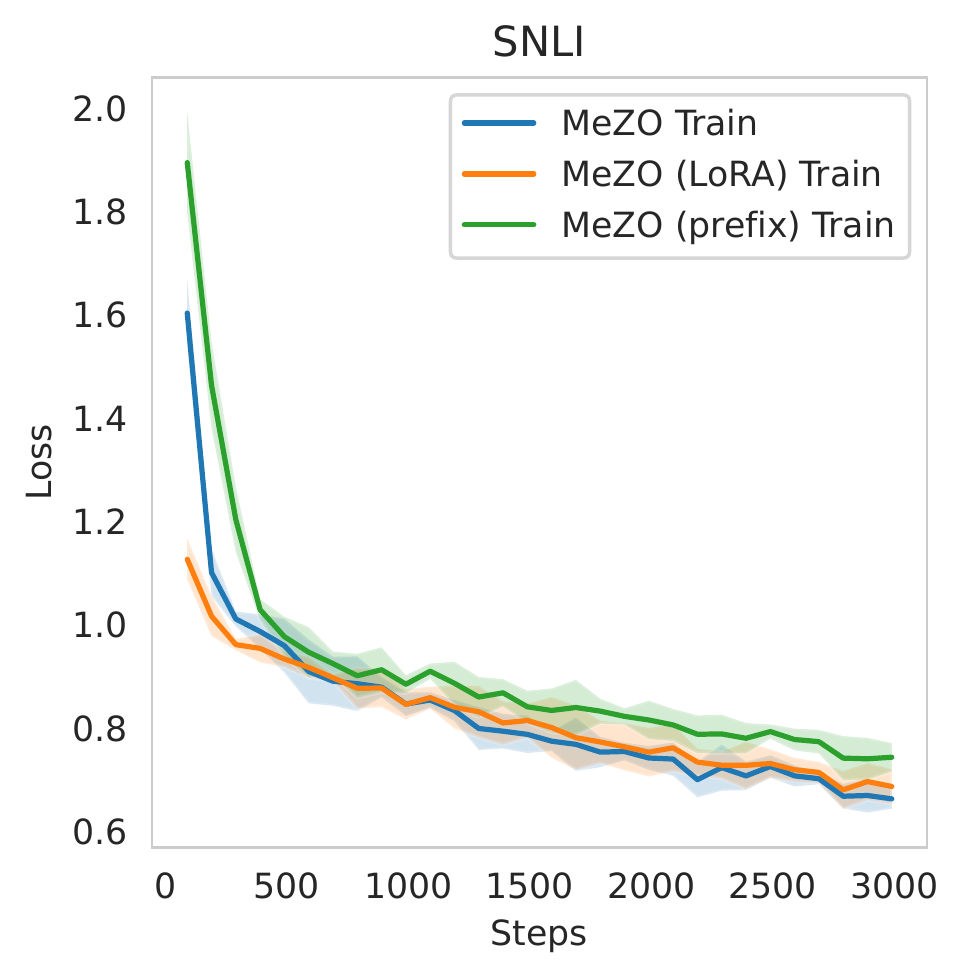}
      \label{fig:peft_plot_snli}
    \end{minipage} 
    \caption{\mezo{} does not optimize significantly faster when tuning fewer parameters, agreeing with our theory in~\Cref{sec:theory}.}
    \label{fig:peft_plot_convg}
  \end{figure}

\subsection{ZO vs BBTv2}
\label{app_sec:bbtv2}

We compare ZO with BBTv2~\citep{sun2022bbtv2}  on mutually assessed tasks in Table~\ref{tab:bbtv2}. 
ZO significantly outperform BBTv2. 
Furthermore, BBTv2 is limited to optimize in low-dimensional space and requires
prefix-tuning and a down-projection to reduce the number of optimized parameters. 
BBTv2 also employs an iterative scheme which only optimizes one layer at a time. 
In contrast, ZO works with both full-parameter tuning and PEFT, as shown in our experiments (\Cref{sec:exp}) and theory~(\Cref{sec:theory}).

\begin{table*}[h]

\centering
\resizebox{0.7\textwidth}{!}{
    \setlength{\tabcolsep}{0.3cm}
    \begin{tabular}{lcccccc}
    \toprule
     Task  &  \multicolumn{1}{c}{\textbf{SST-2}} & \multicolumn{1}{c}{\textbf{SNLI}} & \multicolumn{1}{c}{\textbf{RTE}} \\
    Task type & \multicolumn{1}{c}{------ sentiment ------} & \multicolumn{2}{c}{-- natural language inference --} \\
    \midrule
        Zero-shot & 79.0  & 50.2  & 51.4  \\
\midrule

    BBTv2 & 90.3 (1.7) &  57.3 (2.3) & 56.7 (3.3)\\
    \mezo{} & 90.5 (1.2)	&68.5 (3.9)	& 64.0 (3.3)  \\
    \mezo{} (LoRA) & 91.4 (0.9)  & 69.7 (6.0) & 64.9 (3.6) \\
    \mezo{} (prefix) & 90.8 (1.7)	&	71.6 (2.5) & 65.4 (3.9)	\\

    \bottomrule
    \end{tabular}}
    \caption{
        ZO vs BBTv2 with RoBERTa-large. BBTv2 performance is from \citet{sun2022bbtv2}.
    }
    \label{tab:bbtv2}
\end{table*}

\subsection{Memory profiling}
\label{sec:more_memory}

We show the detailed numbers of memory profiling results 
\Cref{tab:memory_detail}, which also corresponds to \Cref{fig:memory_fig}.
For how we profile the memory usage, 
please refer to \Cref{sec:profiling}.

\begin{table}[h]
\centering
\small
    \begin{tabular}{lcccccc}
    \toprule
    Method  & \tf{Zero-shot / \mezo{}} & \tf{ICL} &   \tf{Prefix FT} & \tf{Full-parameter FT} \\
    \midrule
    1.3B & 1xA100 (4GB)  & 1xA100 (6GB)  & 1xA100 (19GB) & 1xA100 (27GB)\\
    2.7B & 1xA100 (7GB)  & 1xA100 (8GB)  & 1xA100 (29GB) & 1xA100 (55GB)\\
    6.7B & 1xA100 (14GB) & 1xA100 (16GB) & 1xA100 (46GB) & 2xA100 (156GB)\\
    13B & 1xA100 (26GB)  & 1xA100 (29GB) & 2xA100 (158GB) & 4xA100 (316GB)\\
    30B & 1xA100 (58GB)  & 1xA100 (62GB) & 4xA100 (315GB) & 8xA100 (633GB)\\
    66B & 2xA100 (128GB) & 2xA100 (134GB)& 8xA100 & 16xA100\\
    \bottomrule
    \end{tabular}
    \vspace{5pt}
    \caption{
        Memory usage on the MultiRC (avg \#tokens=400) dataset.
    }
    \label{tab:memory_detail}
\end{table}

\subsection{Wallclock time efficiency}
\label{sec:time}

In this section, we measure the wallclock time efficiency of MeZO compared to full-parameter FT, with respect to different model sizes. 
We conduct our experiments with 80GB A100s connected by NVLink and InfiniteBand, which are state-of-the-art solutions for distributed training. 
As shown in Table~\ref{tab:wallclock}, 
on the MultiRC datasets, 
training with MeZO brings 7.74$\times$ speedup per step compared to full-parameter FT on a 30B model. 
This is due to (1) MeZO does not require costly backpropagation and (2) MeZO requires fewer GPUs and reduces the multi-GPU communication overhead. 
We can see that MeZO has a bigger advantage when training larger models---the multi-GPU overhead for fine-tuning is larger.

Note that even though  MeZO has better per-step wallclock efficiency, it requires significantly more steps than standard FT. 
Taking our OPT-30B experiments as an example: MeZO takes 32$\times$ more steps than standard FT, while FT takes 8$\times$ more GPUs and 7.74$\times$ more time per step. Overall, MeZO  requires only half as many GPU-hours as FT for a 30B model.


\begin{table}[h]
\centering
\small
    \begin{tabular}{lccccc}
    \toprule
    & \tf{1.3B} & \tf{2.7B} & \tf{13B} & \tf{30B} & \tf{66B} \\
    \midrule
 MeZO (bsz=16)& 0.815s (1)& 1.400s (1)& 2.702s (1)& 5.896s (1) &12.438s (4)\\
 MeZO (bsz=8)& 0.450s (1)& 0.788s (1)& 1.927s (1)& 4.267s  (1)&7.580s (2)\\
 FT (bsz=8)& 0.784s (1)& 1.326s (1)& 13.638s (4)& 45.608s (8)&84.098s (20)\\
 & bspd=2, ga=4 & bspd=2, ga=4 & bspd=1, ga=2 & bspd=1, ga=1 & bspd=1, ga=1\\
 \bottomrule
    \end{tabular}
    \vspace{5pt}
    \caption{
        Wallclock time per step of different training methods. Numbers in brackets are numbers of GPUs required. It is measured on 80GB A100s with NVLink and InfiniteBand connections. The wallclock time is averaged over 100 steps. It is measured on the MultiRC task with the OPT family. 
        We use a batch size (``bsz'') of 8 for FT and 16 for MeZO (consistant with our main experiment setting). 
        For comparison we also add MeZO with a batch size of 8.
        For FT (FSDP), we show the following additional information. ``bspd'': batch size per device. ``ga'': gradient accumulation steps. The effective batch size is bspd$\times$ga$\times$ \#GPUs.
        Note that for FT 66B, the effective batch size 20.
    }
    \label{tab:wallclock}
\end{table}

\clearpage
\section{Proofs}

\begin{proof}[Proof of \Cref{lem:covariance}]
    We first note that in the $\epsilon \rightarrow 0$ limit, we have
    \begin{align*}
        \hat \nabla \cL(\btheta; \mathcal{B}) = \frac{1}{Bn}\sum_{(\bx, \by) \in \cB}\sum_{i \in [n]}\bz_i\bz_i^\top\nabla \cL(\btheta; \{(\bx, \by)\}).
    \end{align*}
    Taking expectation over the batch $\cB$ and the $\bz_i$, we have $\E[\hat \nabla \cL(\btheta; \mathcal{B})] = \nabla \cL(\btheta)$, so $\hat \nabla \cL(\btheta; \mathcal{B})$ is an unbiased estimator of the gradient. 

    Computing the second moment, we get
    \begin{align*}
        &\E\left[\hat \nabla \cL(\btheta; \mathcal{B})\hat \nabla \cL(\btheta; \mathcal{B})^\top\right]\\
        &\quad = \frac{1}{B^2n^2}\sum_{(\bx_1, \by_1),(\bx_2, \by_2) \in \cB}\sum_{i, j \in [n]}\E\left[(\bz_i\bz_i^\top\nabla \cL(\btheta; \{(\bx_1, \by_1)\}))(\bz_j\bz_j^\top\nabla \cL(\btheta; \{(\bx_2, \by_2)\}))^\top\right]
    \end{align*}

    Let $\bu, \bv$ be two arbitrary vectors. We have that
    \begin{align*}
        \E_{\bz_i, \bz_j}[\bz_i\bz_i^\top\bu\bv^\top\bz_j\bz_j^\top] = \bu\bv^\top
    \end{align*}
    when $i \neq j$, and
    \begin{align*}
        \E_{\bz_i}[\bz_i\bz_i^\top\bu\bv^\top\bz_i\bz_i^\top] &= \E_{\bz}[\bz^{\otimes 4}](\bu, \bv)\\
        &= \frac{3d}{d+2}\Sym(\bI^{\otimes 2})(\bu, \bv)\\
        &= \frac{d}{d+2}\cdot\bu^\top\bv\cdot\bI + \frac{2d}{d+2}\cdot\bu\bv^\top.
    \end{align*}
    Therefore
    \begin{align*}
        &\E\left[\hat \nabla \cL(\btheta; \mathcal{B})\hat \nabla \cL(\btheta; \mathcal{B})^\top\right]\\
        &= \frac{1}{B^2}\sum_{(\bx_1, \by_1),(\bx_2, \by_2) \in \cB} \left(\frac{n - 1}{n} + \frac{2d}{n(d+2)}\right)\E\left[\cL(\btheta; \{(\bx_1, \by_1)\})\cL(\btheta; \{(\bx_2, \by_2)\})^\top \right]\\
        &\qquad + \frac{d}{n(d+2)}\cdot \E\left[\cL(\btheta; \{(\bx_1, \by_1)\})^\top\cL(\btheta; \{(\bx_2, \by_2)\}) \right] \bI.
    \end{align*}

    Next, note that when $(\bx_1, \by_1) \neq (\bx_2, \by_2)$, we have
    \begin{align*}
        \E\left[\cL(\btheta; \{(\bx_1, \by_1)\})\cL(\btheta; \{(\bx_2, \by_2)\})^\top \right] = \nabla \cL(\btheta)\nabla \cL(\btheta)^\top,
    \end{align*}
    and when $(\bx_1, \by_1) = (\bx_2, \by_2)$ we have
    \begin{align*}
        \E\left[\cL(\btheta; \{(\bx_1, \by_1)\})\cL(\btheta; \{(\bx_2, \by_2)\})^\top \right] = \nabla \cL(\btheta)\nabla \cL(\btheta)^\top + \bSigma_{MB}(\btheta).
    \end{align*}
    Therefore
    \begin{align*}
        \frac{1}{B^2}\sum_{(\bx_1, \by_1),(\bx_2, \by_2) \in \cB} \E\left[\cL(\btheta; \{(\bx_1, \by_1)\})\cL(\btheta; \{(\bx_2, \by_2)\})^\top \right] = \nabla \cL(\btheta)\nabla \cL(\btheta)^\top + \frac{1}{B}\bSigma(\btheta),
    \end{align*}
    and plugging this yields
    \begin{align}
        \begin{aligned}\label{eq:ZO_cov}
        \E\left[\hat \nabla \cL(\btheta; \mathcal{B})\hat \nabla \cL(\btheta; \mathcal{B})^\top\right] & = \left(1 + \frac{d-2}{n(d+2)}\right)\cdot\left(\nabla \cL(\btheta)\nabla \cL(\btheta)^\top + \frac{1}{B}\bSigma(\btheta)\right)\\
        &\qquad + \frac{d}{n(d+2)}\bI \cdot\left(\norm{\nabla \cL(\btheta)}^2 + \frac{1}{B}\tr(\bSigma(\btheta)) \right).
        \end{aligned}
    \end{align}
    Finally, we have
    \begin{align*}
        \E\left[\norm{\hat \nabla \cL(\vtheta; \cB)}^2\right] &= \left(1 + \frac{d^2 + d - 2}{n(d+2)}\right)\cdot \left(\norm{\nabla \cL(\btheta)}^2 + \frac{1}{B}\tr(\bSigma(\btheta)) \right)\\
        &= \frac{d + n - 1}{n}\cdot \E\left[\norm{\nabla \cL(\vtheta; \cB)}^2\right].
    \end{align*}
\end{proof}

\begin{proof}[Proof of \Cref{thm:rate_comparison}] 
By Taylor's theorem with remainder, we have that
\begin{align*}
    \cL(\btheta_{t+1}) &= \cL(\btheta_t) + \nabla \cL(\btheta_t)^\top(\btheta_{t+1} - \btheta_t)\\
    &\quad + \int_0^1 \lambda (\btheta_{t+1} - \btheta_t)^\top\nabla^2 \cL(\lambda\btheta_{t+1} + (1-\lambda)\btheta_t)(\btheta_{t+1} - \btheta_t)^\top d\lambda
\end{align*}
Next, note that
\begin{align*}
    \norm{\btheta_{t+1} - \btheta_t} = \eta\norm{\hat\nabla \cL(\vtheta; \cB)} \le \eta\sqrt{d}\cdot \frac{1}{Bn}\sum_{}\abs{\bz_i^\top\nabla \cL(\btheta; \{(\bx, \by)\})} \le \eta dG(\vtheta_t).
\end{align*}
Therefore $\norm{\lambda \btheta_{t+1} + (1 - \lambda)\btheta_t -\btheta_t} \le \eta dG(\vtheta_t)$. By the assumption we have the upper bound $\nabla^2 \cL(\lambda\btheta_{t+1} + (1-\lambda)\btheta_t) \preceq \bH(\btheta_t)$, and thus
\begin{align*}
    \cL(\btheta_{t+1}) &\le \cL(\btheta_t) + \nabla \cL(\btheta_t)^\top(\btheta_{t+1} - \btheta_t) +(\btheta_{t+1} - \btheta_t)^\top\bH(\btheta_t)(\btheta_{t+1} - \btheta_t)\\
    & = \cL(\btheta_t) - \eta\nabla \cL(\btheta_t)^\top \hat \nabla \cL(\btheta_t; \cB) + \frac12\eta^2\hat \nabla \cL(\btheta_t; \cB)^\top\bH(\btheta_t)\hat \nabla \cL(\btheta_t; \cB).\\
\end{align*}
Taking the conditional expectation with respect to $\btheta_t$ and plugging in \eqref{eq:ZO_cov}, the formula for the covariance of our ZO estimate $\hat \nabla \cL(\btheta_t; \cB)$, yields
\begin{align*}
    \E[\cL(\btheta_{t+1}) \mid \btheta_t] &\le \cL(\btheta_t) - \eta\norm{\nabla \cL(\btheta_t)}^2 + \frac{\eta^2}{2}\left\langle \bH(\btheta_t), \E\left[\hat \nabla \cL(\btheta; \mathcal{B})\hat \nabla \cL(\btheta; \mathcal{B})^\top\right] \right\rangle\\
    &= \cL(\btheta_t) - \eta\norm{\nabla \cL(\btheta_t)}^2 + \frac{\eta^2}{2}\cdot\frac{d}{n(d+2)}\left(\norm{\nabla \cL(\btheta_t)}^2 + \frac{1}{B}\tr(\bSigma(\btheta_t))\right)\tr(\bH(\btheta_t))\\
    &\qquad + \frac{\eta^2}{2}\left(1 + \frac{d-2}{n(d+2)}\right)\left(\nabla \cL(\btheta_t)^\top\bH(\btheta_t) \nabla \cL(\btheta_t) + \frac{1}{B}\langle \bSigma(\btheta_t), \bH(\btheta_t)\rangle\right)
\end{align*}

By assumption, the Hessian upper bound $\bH(\btheta_t)$ satisfies $\norm{\bH(\vtheta_t)}_{op} \le \ell$ and $\tr(\bH(\vtheta_t)) \le \ell r$. Thus
\begin{align*}
    \E[\cL(\btheta_{t+1}) \mid \btheta_t] &\le \cL(\btheta_t) - \eta\norm{\nabla \cL(\btheta_t)}^2 + \frac{\eta^2\ell}{2}\cdot\left(\frac{dr + d-2}{n(d+2)} + 1\right)\cdot \left(\norm{\nabla \cL(\btheta_t)}^2 + \frac{1}{B}\tr(\bSigma(\btheta_t))\right)\\
    &= \cL(\btheta_t) - \eta\norm{\nabla \cL(\btheta_t)}^2 + \frac{\eta^2\ell}{2}\cdot\left(\frac{dr + d-2}{n(d+2)} + 1\right)\cdot \E\left[\norm{\nabla \cL(\vtheta_t; \cB)}^2\right],
\end{align*}
as desired.
\end{proof}


\subsection{Proofs of Global Convergence}\label{sec:global_proofs}

\begin{lemma}\label{lem:global_SGD}
    Let $\cL(\btheta)$ be $\mu$-PL and let there exist $\alpha$ such that $\tr(\bSigma(\btheta)) \le \alpha (\cL(\btheta) - \cL^*)$ for all $\btheta$. Then after
    \begin{align*}
        t = O \left(\left(\frac{\ell}{\mu} + \frac{\ell\alpha}{\mu^2B}\right)\log\frac{\cL(\vtheta_0) - \cL^*}{\epsilon}\right)
    \end{align*}
    iterations of \sgd{} we have $\E[\cL(\btheta_t)] \le \cL^* + \epsilon.$
\end{lemma}
\begin{proof}[Proof of \Cref{lem:global_SGD}]
    The descent lemma for SGD yields
    \begin{align*}
    \E[\cL(\btheta_{t+1}) \mid \btheta_t] - \cL(\btheta_t) \le -\eta\norm{\nabla \cL(\btheta_t)}^2 + \frac12\eta^2\ell\cdot\E[\norm{\nabla\cL(\vtheta_t;\cB)}^2].
    \end{align*}
    Plugging in $\E[\norm{\nabla\cL(\vtheta_t;\cB)}^2] = \norm{\nabla \cL(\vtheta_t)}^2 + \frac{1}{B}\tr(\bSigma(\btheta_t))$ and selecting a learning rate $\eta \le \frac{1}{\ell}$ yields
    \begin{align*}
        \E[\cL(\btheta_{t+1}) \mid \btheta_t] \le \cL(\btheta_t) - \frac{\eta}{2}\norm{\nabla \cL(\btheta_t)}^2 + \frac{\eta^2\ell}{2B}\tr(\bSigma(\btheta_t))
    \end{align*}
    Since $\cL$ is $\mu$-PL, we get
    \begin{align*}
        \E[\cL(\btheta_{t+1}) \mid \btheta_t] \le \cL(\btheta_t) - \eta\mu(\cL(\btheta_t) - \cL^*) + \frac{\eta^2\ell}{2B}\tr(\bSigma(\btheta_t)).
    \end{align*}
    Since $\tr(\bSigma(\btheta_t)) \le \alpha (\cL(\btheta_t) - \cL^*)$, we have
    \begin{align*}
        \E[\cL(\btheta_{t+1}) \mid \btheta_t] \le \cL(\btheta_t) - \eta\mu(\cL(\btheta_t) - \cL^*) + \frac{\eta^2\ell\alpha}{2B}(\cL(\btheta_t) - \cL^*).
    \end{align*}
    Altogether,
    \begin{align*}
        \E[\cL(\btheta_{t+1})] - \cL^* \le \left(1 - \eta\mu + \frac{\eta^2\ell\alpha}{2B}\right)(\E[\cL(\btheta_t)] - \cL^*)
    \end{align*}
    Choosing $\eta = \min(\frac{1}{\ell}, \frac{\mu B}{\ell \alpha})$, we obtain
    \begin{align*}
        \E[\cL(\btheta_{t+1})] - \cL^* \le \left(1 - \min(\frac{\mu}{2\ell}, \frac{\mu^2B}{2\ell\alpha})\right)(\E[\cL(\btheta_t)] - \cL^*).
    \end{align*}
    Therefore we reach a solution with $\E[\cL(\btheta_{t})] - \cL^* \le \epsilon$ after
    \begin{align*}
        t = \max\left(\frac{2\ell}{\mu}, \frac{2\ell\alpha}{\mu^2B}\right)\log\left(\frac{\cL(\btheta_0) - \cL^*}{\epsilon}\right) = O \left(\left(\frac{\ell}{\mu} + \frac{\ell\alpha}{\mu^2B}\right)\log\frac{\cL(\vtheta_0) - \cL^*}{\epsilon}\right)
    \end{align*} iterations.
\end{proof}

\begin{proof}[Proof of \Cref{lem:global_ZO-SGD}]
    By \Cref{cor:rate_comparison}, ZO-SGD with $\etazo = \gamma^{-1}\etasgd$ yields
    \begin{align*}
    \E[\cL(\btheta_{t+1}) \mid \btheta_t] - \cL(\btheta_t) \le \frac{1}{\gamma}\cdot\left[-\etasgd\norm{\nabla \cL(\btheta_t)}^2 + \frac12\etasgd^2\ell\cdot\E[\norm{\nabla\cL(\vtheta;\cB)}^2] \right].
    \end{align*}
    As in the proof for SGD, choosing $\etasgd \le \frac{1}{\ell}$ yields
    \begin{align*}
        \E[\cL(\btheta_{t+1}) \mid \btheta_t] -\cL(\btheta_t) \le  \gamma^{-1}\cdot\left[- \frac{\etasgd}{2}\norm{\nabla \cL(\btheta_t)}^2 + \frac{\etasgd^2\ell}{2B}\tr(\bSigma(\btheta_t))\right].
    \end{align*}
    Therefore under $\mu$-PL and the $\tr(\bSigma(\btheta_t)) \le \alpha (\cL(\btheta_t) - \cL^*)$ assumption we obtain
    \begin{align*}
        \E[\cL(\btheta_{t+1})] -\E[\cL(\btheta_t)] &\le \gamma^{-1}\cdot \left[-\etasgd\mu  + \frac{\etasgd^2\ell\alpha}{2B}\right]\cdot(\E[\cL(\btheta_t)] - \cL^*)\\
        \Longrightarrow \E[\cL(\btheta_{t+1})] - \cL^* &\le \left(1 - \gamma^{-1}\left(\etasgd\mu - \frac{\etasgd^2\ell\alpha}{2B}\right)\right)(\E[\cL(\btheta_t)] - \cL^*).
    \end{align*}
    Choosing $\etasgd = \min(\frac{1}{\ell}, \frac{\mu B}{\ell \alpha})$ yields
        \begin{align*}
        \E[\cL(\btheta_{t+1})] - \cL^* \le \left(1 - \gamma^{-1}\cdot\min(\frac{\mu}{2\ell}, \frac{\mu^2B}{2\ell\alpha})\right)(\E[\cL(\btheta_t)] - \cL^*).
    \end{align*}
        Therefore we reach a solution with $\E[\cL(\btheta_{t})] - \cL^* \le \epsilon$ after
    \begin{align*}
        t = \gamma\cdot\max\left(\frac{2\ell}{\mu}, \frac{2\ell\alpha}{\mu^2B}\right)\log\left(\frac{\cL(\btheta_0) - \cL^*}{\epsilon}\right) = \cO \left(\left(\frac{r}{n} + 
        1 \right)\cdot\left(\frac{\ell}{\mu} + \frac{\ell\alpha}{\mu^2B}\right)\log\frac{\cL(\vtheta_0) - \cL^*}{\epsilon}\right)
    \end{align*} iterations.
\end{proof}

\subsubsection{Verification of assumptions}

We show that the $\tr(\bSigma(\btheta_t)) \le \alpha (\cL(\btheta_t) - \cL^*)$ assumption holds for certain losses.

First, consider optimizing the model $f(\bx; \btheta)$ with square loss, so that
\begin{align*}
    \cL(\btheta) = \frac{1}{N}\sum_{i \in [N]}(f(\bx_i; \btheta) - \by_i)^2.
\end{align*}
One then has that
\begin{align*}
    \bSigma(\btheta) = \frac{2}{N}\sum_{i \in [N]}(f(\bx_i; \btheta) - \by_i)^2\nabla f(\bx_i; \btheta)\nabla f(\bx_i; \btheta)^\top - \nabla \cL(\btheta)\nabla \cL(\btheta)^\top.
\end{align*}
Therefore
\begin{align*}
    \tr(\bSigma(\btheta)) &\le \frac{2}{N}\sum_{i \in [N]}(f(\bx_i; \btheta) - \by_i)^2\norm{\nabla f(\bx_i; \btheta)}^2\\
    &\le 2\cL(\btheta)\sum_{i \in [N]}\norm{\nabla f(\bx_i; \btheta)}^2.
\end{align*}
Assume that the data can be interpolated, i.e., $\cL^* = 0$. If the function is $L$-Lipschitz, i.e., $\norm{\nabla f(\bx; \btheta)} \le L$, then the condition holds with $\alpha = 2NL^2$. If we are in the kernel regime, i.e., $f(\bx_i; \btheta) = \phi(\bx_i)^\top\btheta$ for some feature map $\phi$, then
\begin{align*}
    \nabla^2\cL(\btheta) = \frac{2}{N}\sum_{i \in [N]}f(\bx_i; \btheta)\nabla f(\bx_i; \btheta)^\top.
\end{align*}
Thus
\begin{align*}
    \tr(\bSigma(\btheta)) \le N\tr(\nabla^2\cL(\btheta))\cdot \cL(\btheta) \le N\ell r\cdot \cL(\btheta).
\end{align*}
So the condition holds for $\alpha = N\ell r$.

Next, consider the cross entropy loss function, i.e
\begin{align*}
    \cL(\btheta) = \frac{1}{N}\sum_{i \in [N]}\exp(-y_if(\bx_i; \vtheta)).
\end{align*}
One then has that
\begin{align*}
    \bSigma(\vtheta) = \frac{1}{N}\sum_{i \in [N]}\exp(-2y_if(\vx_i; \vtheta))y_i^2\nabla f(\vx_i; \vtheta)\nabla f(\vx_i; \vtheta)^\top - \cL(\vtheta)\cL(\vtheta)^\top ,
\end{align*}
Assume that the targets $y_i$ are bounded in $[-1, 1]$ (which is true for binary classification tasks), and that $\cL^* = 0$ (which can be achieved if $\abs{f(\vx; \vtheta)}$ can be sent to $\infty$) we have that
\begin{align*}
    \tr(\bSigma(\vtheta)) \le \frac{1}{N}\sum_{i \in [N]}\exp(-2y_if(\vx_i; \vtheta))\norm{\nabla f(\vx_i; \vtheta)}^2.
\end{align*}
In the kernel regime, $f(\bx_i; \btheta) = \phi(\bx_i)^\top\vtheta$, and thus
\begin{align*}
    \nabla^2 \cL(\vtheta) = \frac{1}{N}\sum_{i \in [N]}\exp(-y_if(\vx_i; \vtheta))\nabla f(\vx_i; \vtheta)\nabla f(\vx_i; \vtheta)^\top.
\end{align*}
Therefore
\begin{align*}
    \tr(\bSigma(\vtheta)) \le N\tr(\nabla^2\cL(\btheta))\cdot \cL(\btheta) \le N\ell r\cdot \cL(\btheta).
\end{align*}
Therefore the condition holds with $\alpha = N\ell r$ as well.

\subsection{Proofs for Gaussian perturbations}\label{sec:app_gaussian_z}

The first lemma computes the second moment of the covariance estimate $\hat \nabla \cL(\btheta; \cB)$ when $\bz$ is drawn $\cN(0, \bI)$.
\begin{lemma} Let $\bz_i \sim \cN(0, \bI)$ i.i.d. Then
    \begin{align}
        \begin{aligned}\label{eq:ZO_cov_gaussian}
        \E\left[\hat \nabla \cL(\btheta; \mathcal{B})\hat \nabla \cL(\btheta; \mathcal{B})^\top\right] & = \left(1 + \frac{1}{n}\right)\cdot\left(\nabla \cL(\btheta)\nabla \cL(\btheta)^\top + \frac{1}{B}\bSigma_{MB}(\btheta)\right)\\
        &\qquad + \frac{1}{n}\bI \cdot\left(\norm{\nabla \cL(\btheta)}^2 + \frac{1}{B}\tr(\bSigma_{MB}(\btheta)) \right).
        \end{aligned}
    \end{align}
\end{lemma}
\begin{proof}
    As in the proof of \Cref{lem:covariance}, we have that in the $\epsilon \rightarrow 0$ limit
    \begin{align*}
        &\E\left[\hat \nabla \cL(\btheta; \mathcal{B})\hat \nabla \cL(\btheta; \mathcal{B})^\top\right]\\
        &\quad = \frac{1}{B^2n^2}\sum_{(\bx_1, \by_1),(\bx_2, \by_2) \in \cB}\sum_{i, j \in [n]}\E\left[(\bz_i\bz_i^\top\nabla \cL(\btheta; \{(\bx_1, \by_1)\}))(\bz_j\bz_j^\top\nabla \cL(\btheta; \{(\bx_2, \by_2)\}))^\top\right]
    \end{align*}

    For vectors $\bu, \bv$, we have that
    \begin{align*}
        \E_{\bz_i, \bz_j}[\bz_i\bz_i^\top\bu\bv^\top\bz_j\bz_j^\top] = \bu\bv^\top
    \end{align*}
    when $i \neq j$, and
    \begin{align*}
        \E_{\bz_i}[\bz_i\bz_i^\top\bu\bv^\top\bz_i\bz_i^\top] = \E_{\bz}[\bz^{\otimes 4}](\bu, \bv) = 3\Sym(\bI^{\otimes 2})(\bu, \bv) = \bu^\top\bv\cdot\bI + 2\bu\bv^\top.
    \end{align*}
    Therefore
    \begin{align*}
        &\E\left[\hat \nabla \cL(\btheta; \mathcal{B})\hat \nabla \cL(\btheta; \mathcal{B})^\top\right]\\
        &= \frac{1}{B^2}\sum_{(\bx_1, \by_1),(\bx_2, \by_2) \in \cB} \left(\frac{n - 1}{n} + \frac{2}{n}\right)\E\left[\cL(\btheta; \{(\bx_1, \by_1)\})\cL(\btheta; \{(\bx_2, \by_2)\})^\top \right]\\
        &\qquad + \frac{1}{n}\cdot \E\left[\cL(\btheta; \{(\bx_1, \by_1)\})^\top\cL(\btheta; \{(\bx_2, \by_2)\}) \right] \bI.
    \end{align*}
    In the proof of \Cref{lem:covariance} we showed that
    \begin{align*}
        \frac{1}{B^2}\sum_{(\bx_1, \by_1),(\bx_2, \by_2) \in \cB} \E\left[\cL(\btheta; \{(\bx_1, \by_1)\})\cL(\btheta; \{(\bx_2, \by_2)\})^\top \right] = \nabla \cL(\btheta)\nabla \cL(\btheta)^\top + \frac{1}{B}\bSigma(\btheta).
    \end{align*}
    Plugging this yields
    \begin{align}
        \begin{aligned}\label{eq:ZO_cov}
        \E\left[\hat \nabla \cL(\btheta; \mathcal{B})\hat \nabla \cL(\btheta; \mathcal{B})^\top\right] & = \left(\frac{n+1}{n}\right)\cdot\left(\nabla \cL(\btheta)\nabla \cL(\btheta)^\top + \frac{1}{B}\bSigma(\btheta)\right)\\
        &\qquad + \frac{1}{n}\bI \cdot\left(\norm{\nabla \cL(\btheta)}^2 + \frac{1}{B}\tr(\bSigma(\btheta)) \right).
        \end{aligned}
    \end{align}
\end{proof}

We can prove an analog to \Cref{thm:rate_comparison} in the case where the $\bz_i$ are Gaussian. One challenge is that $\norm{\btheta_{t+1} - \btheta_t}$ is no longer bounded; instead we the $r$-local effective rank assumption only holds with high probability, and thus to bound the expected loss decrease we must control the probability of the $\norm{\btheta_{t+1} - \btheta_t}$ being large.

Consider the following modified version of the local $r$-effective rank assumption, where the upper bound on the Hessian is measured over a ball of radius twice as large as the one in~\Cref{assume:low_eff_rank}.

\begin{assumption}[Local $r$-effective rank, Gaussian]\label{assume:low_eff_rank_gauss}
Let $G(\btheta_t) = \max_{(\bx, \by) \in \cD}\norm{\nabla\cL(\btheta_t;\{(\bx, \by)\})}$. There exists a matrix $\bH(\btheta_t)$ such that:
\begin{enumerate}
    \item For all $\vtheta$ such that $\norm{\vtheta - \vtheta_t} \le 2\eta d G(\vtheta_t)$, we have $\nabla^2 \cL(\btheta) \preceq \bH(\btheta_t)$.
    \item The effective rank of $\bH(\btheta_t)$, i.e., $\tr(\bH(\btheta_t))/\norm{\bH(\btheta_t)}_{op}$, is at most $r$. 
\end{enumerate}
\end{assumption}

\begin{theorem}[Dimension-Free Rate, Gaussian $\vz$]\label{thm:rate_comparison_gauss}
Assume the loss exhibits local $r$-effective rank (\Cref{assume:low_eff_rank_gauss}). If $\btheta_{t+1} = \btheta_t - \etazo \hat \nabla \cL(\btheta_t; \cB)$ is a single step of \zosgd{} using the $n$-SPSA estimate with a minibatch of size $B$, then there exists a $\gamma = \Theta(r/n)$ such that the expected loss decrease can be bounded as
\begin{align*}
    &\E[\cL(\btheta_{t+1}) \mid \btheta_t] - \cL(\btheta_t)\\
    &\quad \le - \etazo\norm{\nabla \cL(\btheta_t)}^2 + \frac12\etazo^2\ell\cdot \gamma \cdot\E[\norm{\nabla\cL(\vtheta_t;\cB)}^2] + \etazo^2\ell G(\btheta_t)^2\exp(-\Omega(nd)).
\end{align*}
\end{theorem}

\begin{proof}[Proof of \Cref{thm:rate_comparison_gauss}]
    Let $\cA$ be the event that $\norm{\vtheta_{t+1} - \vtheta_t} \le 2\eta d G(\vtheta_t)$. On $\cA$, we have that
    \begin{align*}
        \cL(\vtheta_{t+1}) \le \cL(\vtheta_t) - \eta \nabla \cL(\vtheta_t)^\top\hat \nabla \cL(\vtheta; \cB) + \frac12\eta^2\hat \nabla \cL(\vtheta_t; \cB)^\top \bH(\vtheta) \hat \nabla \cL(\vtheta_t; \cB).
    \end{align*}
    Likewise, since $\cL$ is $\ell$-smooth, we have that
    \begin{align*}
        \cL(\vtheta_{t+1}) \le \cL(\vtheta_t) - \eta \nabla \cL(\vtheta_t)^\top\hat \nabla \cL(\vtheta; \cB) + \frac12\eta^2\ell \norm{\hat \nabla \cL(\vtheta_t; \cB)}^2.
    \end{align*}
    Therefore
    \begin{align*}
        \E[\cL(\vtheta_{t+1}) \mid \vtheta_t] &\le \cL(\vtheta_{t+1}) - \eta\norm{\nabla \cL(\vtheta_t)}^2 + \frac12\eta^2\left\langle \E\left[\hat \nabla \cL(\vtheta; \cB)\hat \nabla \cL(\vtheta; \cB)^\top \cdot \mathbf{1}(\cA)\right], \bH(\vtheta_t) \right\rangle\\
        &\qquad + \frac12\eta^2 \ell \E\left[\norm{\hat \nabla \cL(\vtheta_t; \cB)}^2 \cdot \mathbf{1}(\neg \cA) \right]\\
        &= \cL(\vtheta_{t+1}) - \eta\norm{\nabla \cL(\vtheta_t)}^2 + \frac12\eta^2\left\langle \E\left[\hat \nabla \cL(\vtheta; \cB)\hat \nabla \cL(\vtheta; \cB)^\top \right], \bH(\vtheta_t) \right\rangle\\
        &\qquad \frac12\eta^2\left\langle \E\left[\hat \nabla \cL(\vtheta; \cB)\hat \nabla \cL(\vtheta; \cB)^\top \cdot \mathbf{1}(\neg \cA)\right], \ell I - \bH(\vtheta_t) \right\rangle.
    \end{align*}
    The latter term can be bounded as follows
    \begin{align*}
        \frac12\eta^2\left\langle \E\left[\hat \nabla \cL(\vtheta; \cB)\hat \nabla \cL(\vtheta; \cB)^\top \cdot \mathbf{1}(\neg \cA)\right], \ell I - \bH(\vtheta_t) \right\rangle &\le \eta^2\ell \E\left[\norm{\hat \nabla \cL(\vtheta; \cB)}^2 \cdot \mathbf{1}(\neg \cA)\right]\\
        &\le \eta^2\ell \E\left[\norm{\hat \nabla \cL(\vtheta; \cB)}^4\right]^{\frac12}\Pr[\neg \cA]^{1/2}.
    \end{align*}
    The gradient estimate $\hat \nabla \cL(\vtheta; \cB)$ satisfies
    \begin{align*}
        \norm{\hat \nabla \cL(\vtheta; \cB)} \le \frac{1}{n}\sum_{i \in [n]}\abs{\vz_i^\top\nabla \cL(\vtheta; \cB)}\cdot\norm{\vz_i}
    \end{align*}
    The expectation term is upper bounded as
    \begin{align*}
        \E\left[\norm{\hat \nabla \cL(\vtheta; \cB)}^4\right] &\le \frac{1}{n}\sum_{i \in [n]}\E\left[\abs{\vz^\top\nabla \cL(\vtheta; \cB)}^4\cdot\norm{\vz}^4 \right]\\
        &\le \E\left[\abs{\vz^\top\nabla \cL(\vtheta; \cB)}^8\right]^{1/2}\E\left[\norm{\vz}^8\right]^{1/2}\\
        &\le \sqrt{105}(d + 6)^2G(\vtheta_t)^4,
    \end{align*}
    where we have plugged in explicit formulas for moments of Gaussian and $\chi^2$ random variables.

    Next, note that on the event $\neg \cA$, we have
    \begin{align*}
        2\eta d G(\vtheta_t) \le \norm{\vtheta_{t+1} - \vtheta_t} = \eta \norm{\hat \nabla \cL(\vtheta_t; \cB)} \le \eta \cdot \frac{1}{n}\sum_{i \in [n]}\norm{\vz_i}^2G(\vtheta_t).
    \end{align*}
    Therefore
    \begin{align*}
        \Pr[\neg \cA] \le \Pr\left[\sum_{i \in [n]}\norm{\bz_i}^2 \ge 2nd\right]
    \end{align*}
    \begin{lemma}[Standard $\chi^2$-tail bound]
        Let $Z$ be a $\chi^2$ random variable with $k$ degrees of freedom. Then
        \begin{align*}
            \Pr[Z \ge k + u] \le \exp\left(-\min\left(\frac{u^2}{16k}, \frac{u}{16}\right)\right)
        \end{align*}
    \end{lemma}
    Since $\sum_{i \in [n]}\norm{\bz_i}^2$ is a $\chi^2$ random variable with $nd$ degrees of freedom, we thus have that
    \begin{align*}
        \Pr[\neg \cA] \le \exp\left(-\frac{nd}{16}\right).
    \end{align*}
    Altogether,
    \begin{align*}
        \frac12\eta^2\left\langle \E\left[\hat \nabla \cL(\vtheta; \cB)\hat \nabla \cL(\vtheta; \cB)^\top \cdot \mathbf{1}(\neg \cA)\right], \ell I - \bH(\vtheta_t) \right\rangle &\le \eta^2\ell 105^{1/4}(d + 6)G(\btheta_t)^2\exp(-\frac{nd}{32})\\
        &= \eta^2 \ell G(\btheta_t)^2 \exp(-\Omega(nd)).
    \end{align*}

    Finally, plugging in \eqref{eq:ZO_cov_gaussian}, along with the fact that $\norm{\bH(\btheta_t)}_{op} \le \ell$ and $\tr(\bH(\btheta_t)) \le \ell r$,
    \begin{align*}
        \left\langle \E\left[\hat \nabla \cL(\vtheta; \cB)\hat \nabla \cL(\vtheta; \cB)^\top \right], \bH(\vtheta_t) \right\rangle &= \frac{r + n+1}{n}\cdot \ell\left(\norm{\nabla \cL(\btheta_t)}^2 + \frac{1}{B}\tr(\bSigma(\btheta_t))\right)\\
        &= \frac{r + n + 1}{n}\cdot \E\left[\norm{\nabla \cL(\vtheta_t; \cB)}^2 \right]
    \end{align*}
    Thus letting $\gamma = \frac{r + n + 1}{n}$ yields
    \begin{align*}
        &\E[\cL(\btheta_{t+1}) \mid \btheta_t] - \cL(\btheta_t)\\
        &\quad \le - \eta\norm{\nabla \cL(\btheta_t)}^2 + \frac12\eta^2\ell\cdot \gamma \cdot\E[\norm{\nabla\cL(\vtheta_t;\cB)}^2] + \eta^2\ell G(\btheta_t)^2\exp(-\Omega(nd)),
    \end{align*}
    as desired.
\end{proof}

\end{document}